\documentclass[twoside,11pt]{article}

\usepackage{blindtext}

\usepackage{jmlr2e}

\usepackage{algorithm,algorithmic}
\usepackage{hyperref,comment}

\usepackage{amsmath,amsfonts,bm}

\def\eqref#1{equation~\ref{#1}}
\def\1{\bm{1}}

\DeclareMathAlphabet{\mathsfit}{\encodingdefault}{\sfdefault}{m}{sl}
\SetMathAlphabet{\mathsfit}{bold}{\encodingdefault}{\sfdefault}{bx}{n}

\def\gA{{\mathcal{A}}}

\def\gD{{\mathcal{D}}}

\def\gF{{\mathcal{F}}}

\def\gH{{\mathcal{H}}}

\def\gM{{\mathcal{M}}}
\def\gN{{\mathcal{N}}}

\def\gP{{\mathcal{P}}}
\def\gQ{{\mathcal{Q}}}
\def\gR{{\mathcal{R}}}
\def\gS{{\mathcal{S}}}
\def\gT{{\mathcal{T}}}

\def\gV{{\mathcal{V}}}

\def\gX{{\mathcal{X}}}

\def\gZ{{\mathcal{Z}}}

\def\sR{{\mathbb{R}}}

\newcommand{\E}{\mathbb{E}}

\newcommand{\R}{\mathbb{R}}

\DeclareMathOperator*{\argmax}{arg\,max}
\DeclareMathOperator*{\argmin}{arg\,min}

\newcommand{\smsigma}{\sigma}
\newcommand{\mainalgoname}{\text{RL from Multi-Source Imperfect Preferences}}
\newcommand{\mainalgoabrv}{\text{RL-MSIP}}
\newcommand{\alg}{\text{Alg}}

\newcommand{\indicate}{\mathbf{1}}

\newcommand{\hatp}{\hat{P}}

\newcommand{\regret}{\operatorname{Reg}}
\newcommand{\Gammaa}{\hat{\Gamma}}

\newcommand{\dr}{D_1^{\tau \vert \Gamma_N}}
\newcommand{\drsum}{D_2^{\Gamma_N}}

\newcommand{\betar}{\beta^{\text{R}}}

\newcommand{\bonr}{b^{\text{R}}}
\newcommand{\dt}{D_{3}^{(s,a) \vert \Gamma_{N,h}}}
\newcommand{\dtsum}{D_{4}^{\Gamma_{N,h}}}

\newcommand{\betap}{\beta^{\text{P}}}

\newcommand{\bonp}{b^{\text{P}}}
\newcommand{\gv}{\gV_h}

\newcommand{\alphar}{\alpha^{\text{R}}}
\newcommand{\alphap}{\alpha^{\text{P}}}

\newcommand{\weightrt}{w_{1,t(\tau,k)}^{\text{R}}}

\newcommand{\weightrrk}{w_{2,k}^{\text{R}}}

\newcommand{\weightpt}{w_{3,t(s,a,k)}^{\text{P}}}

\newcommand{\weightppk}{w_{4,k,h}^{\text{P}}}
\newcommand{\siga}{g}

\newcommand{\UCB}{\operatorname{UCB}}
\newcommand{\Fcal}{\mathcal{F}}
\DeclareMathOperator*{\arginf}{arginf}
\newcommand{\Imp}{\operatorname{Imp}}

\usepackage{lastpage}
\ShortHeadings{Regret Bounds for RL from Multi-Source Imperfect Preferences}{Shi, Liang, Shroff, and Swami}
\firstpageno{1}

\begin{document}

\title{Regret Bounds for Reinforcement Learning from Multi-Source Imperfect Preferences}

\author{\name Ming Shi \email mshi24@buffalo.edu \\
       \addr Department of Electrical Engineering\\
       University at Buffalo\\
       Buffalo, NY 14260, USA
       \AND
       \name Yingbin Liang \email liang.889@osu.edu \\
       \addr Department of Electrical and Computer Engineering\\
       The Ohio State University\\
       Columbus, OH 43210, USA
       \AND
       \name Ness B. Shroff \email shroff.11@osu.edu \\
       \addr Department of Electrical and Computer Engineering and Computer Science and Engineering\\
       The Ohio State University\\
       Columbus, OH 43210, USA
       \AND
       \name Ananthram Swami \email ananthram.swami.civ@army.mil \\
       \addr DEVCOM Army Research Laboratory\\
       Adelphi, MD 20783, USA}

\editor{My editor}

\maketitle

\begin{abstract}%
Reinforcement learning from human feedback (RLHF) replaces hard-to-specify rewards with pairwise trajectory preferences, yet regret-oriented theory often assumes that preference labels are generated consistently from a single ground-truth objective. In practical RLHF systems, however, feedback is typically \emph{multi-source} (annotators, experts, reward models, heuristics) and can exhibit systematic, persistent mismatches due to subjectivity, expertise variation, and annotation/modeling artifacts. We study episodic RL from \emph{multi-source imperfect preferences} through a cumulative imperfection budget: for each source, the total deviation of its preference probabilities from an ideal oracle is at most $\omega$ over $K$ episodes. We propose a unified algorithm with regret $\tilde{O}(\sqrt{K/M}+\omega)$, which exhibits a best-of-both-regimes behavior: it achieves $M$-dependent statistical gains when imperfection is small (where $M$ is the number of sources), while remaining robust with unavoidable additive dependence on $\omega$ when imperfection is large. We complement this with a lower bound $\tilde{\Omega}(\max\{\sqrt{K/M},\omega\})$, which captures the best possible improvement with respect to $M$ and the unavoidable dependence on $\omega$, and a counterexample showing that na\"ively treating imperfect feedback as oracle-consistent can incur regret as large as $\tilde{\Omega}(\min\{\omega\sqrt{K},K\})$. Technically, our approach involves imperfection-adaptive weighted comparison learning, value-targeted transition estimation to control hidden feedback-induced distribution shift, and sub-importance sampling to keep the weighted objectives analyzable, yielding regret guarantees that quantify when multi-source feedback provably improves RLHF and how cumulative imperfection fundamentally limits it.
\end{abstract}


\tableofcontents

% ==============================================================================================================================================================================
\section{Introduction}\label{sec:introduction}

Reinforcement learning (RL) learns a policy through interaction to optimize long-term performance, typically formalized as maximizing expected cumulative return. In standard RL, learning is guided by a designer-specified reward function, e.g.,~\cite{kaelbling1996reinforcement,sutton2018reinforcement,agarwal2019reinforcement}. In many modern applications, however, specifying a reward that faithfully encodes the intended objective is difficult or even infeasible, notably in large language models (LLMs)~\citep{achiam2023gpt,touvron2023llama,anil2023palm,guo2025deepseek}, robotics~\citep{christiano2017deep}, autonomous driving~\citep{sun2024optimizing}, and generative modeling~\citep{lee2023aligning}. This motivates learning from feedback signals that are easier to elicit than an explicit reward.

Reinforcement learning from human feedback (RLHF) uses human judgments, most commonly pairwise comparisons between two trajectories (or responses), as the training signal for alignment~\citep{kaufmann2023survey,casper2023open}. A canonical pipeline collects preference labels, fits a reward/preference model, and then optimizes a policy using the learned model. Recent theory has begun to characterize sample efficiency and regret in preference-based RL and RLHF, under the assumption that feedback is generated consistently from a single ground-truth objective with a stationary noise model, e.g.,~\cite{chatterji2021theory,chen2022human,wang2023rlhf,zhu2023principled,du2024exploration}.

In practice, RLHF feedback rarely comes from a single homogeneous oracle. Labels may be produced by multiple human annotators, domain experts, AI-based feedback models, heuristics, or automated evaluation tools~\citep{leerlaif,tjuatja2024llms,bukharin2024r3m,chowdhury2024rdpo}. These sources can differ systematically due to subjectivity, expertise, instruction framing, calibration drift, and modeling artifacts, leading to preference probabilities that persistently deviate from the idealized objective one would use to define ``ground-truth'' performance. Empirical studies report substantial disagreement and bias across sources and highlight the need to reason about feedback imperfection~\citep{peng2022investigations,santurkar2023whose,yu2024large,chakraborty2024maxmin}. A parallel robustness literature addresses other failure modes in RLHF and preference optimization, such as sparse outlier corruptions in preference labels~\citep{bukharin2024r3m}, reward/preference model misspecification and variance reduction in policy optimization~\citep{ye2025vrpo}, reward-robust objectives under reward-model uncertainty~\citep{yan2024reward}, and distribution shift across prompts or tasks~\citep{mandal2025dro}. However, a regret-oriented theory that explicitly models \textbf{multi-source preference imperfection}, allows history-dependent deviations, and yields a best-of-both-regimes characterization in terms of $(K,M,\omega)$ remains underdeveloped, where $K$ is the number of episodes, $M$ is the number of sources, and $\omega$ quantifies the cumulative deviation from an ideal preference oracle.

We consider episodic MDP with trajectory-level objectives. In each episode $k$, the learner executes $\pi_k$ to generate a trajectory $\tau_k\sim(P^*,\pi_k)$ and compares it against an independent reference-policy rollout $\tau_{k,0}\sim(P^*,\pi_0)$, where $P^*$ is the unknown transition kernel. An ideal preference oracle prefers $\tau_k$ over $\tau_{k,0}$ with probability $p_k^*=\smsigma(R^*(\tau_k)-R^*(\tau_{k,0}))$, where $R^*$ is an unknown trajectory-level reward functional and $\smsigma$ is a known link function. The learner observes $M$ binary labels per episode from $M$ sources, each inducing its own preference probability $p_k^m$. We quantify \emph{multi-source imperfection} by a cumulative deviation budget $\sum_{k=1}^K |p_k^m-p_k^*| \le \omega$ for all $m\in[M]$, which permits heterogeneous deviations while ruling out unbounded adversarial drift. Since only comparisons are observed, our learning target is naturally the induced comparison function $q_R(\tau,\tilde\tau) = \smsigma(R(\tau)-R(\tilde\tau))$.

The imperfection budget creates a fundamental tension between robustness and statistical efficiency. When $\omega$ is large, na\"ively treating labels as unbiased can accumulate systematic error and can lead to regret that scales linearly with $K$. When $\omega$ is small, pooling $M$ labels per episode should sharpen preference estimation, so the regret should improve with $M$. Reconciling these behaviors requires a \emph{single unified algorithm} with guarantees that capture both limits. A central goal of this paper is to characterize the transition between the two regimes and develop a method whose regret bound smoothly interpolates between the low- and high-imperfection cases. Our main contributions are as follows.
\begin{itemize}
\item \textbf{A best-of-both-regimes algorithm and regret upper bound.} We propose \mainalgoname~(\mainalgoabrv) and prove regret $\tilde{O}\!\big(\sqrt{K/M}+\omega\big)$, yielding a unified best-of-both-regimes guarantee: $M$-dependent statistical gains when imperfection is small, and robustness with unavoidable $\omega$ dependence when imperfection is large. At a high level, our method learns a confidence set for the comparison function from multi-source imperfect preference, performs optimistic planning under model estimation error, and uses filtering to keep the learning analyzable.

\item \textbf{A fundamental lower bound.} We establish a worst-case regret lower bound $\tilde{\Omega}\!\big(\max\{\sqrt{K/M},\,\omega\}\big)$, showing that the optimal dependence on the number of sources $M$ is at best $\sqrt{K/M}$, even when imperfection is negligible, and additive dependence on the cumulative imperfection budget $\omega$ is information-theoretically necessary.

\item \textbf{A counterexample for ``ignoring imperfection.''} We exhibit an instance where an unweighted baseline (that pools all labels and treats imperfect feedback as unbiased) incurs $\tilde{\Omega}\!\big(\min\{\omega\sqrt{K},\,K\}\big)$ regret, showing that na\"ively aggregating multi-source labels can amplify cumulative imperfection into a substantially worse rate (e.g., yielding an $\omega\sqrt{K}$ term rather than the desired additive $\omega$ dependence).

\item \textbf{Technical analysis.} Our method involves (i) imperfection-adaptive weighted comparison learning that controls the imperfection cross-term; (ii) value-targeted transition estimation to handle the hidden feedback-induced distribution shift; (iii) policy-level optimism tailored to preference-only objectives; and (iv) sub-importance sampling to control sensitivity so the weighted objectives admit a sharp regret analysis.
\end{itemize}

Our bounds provide actionable guidance for RLHF system design under imperfect preference feedback. They quantify when collecting more feedback sources yields statistical gains (low imperfection) and when such gains saturate due to irreducible source mismatch (high imperfection), pinpointing the transition around $\omega=\tilde{\Theta}(\sqrt{K/M})$. They also explain why unweighted aggregation can amplify cumulative imperfection into an $\omega\sqrt{K}$ term, motivating imperfection-aware aggregation in practice.

In summary, Section~\ref{sec:problemformulation} formalizes the episodic RLHF setting with multi-source imperfect preferences and defines the regret. Section~\ref{sec:algorithmdesign} presents our algorithm \mainalgoabrv. Section~\ref{sec:theoreticalanalysis} gives the main regret guarantees, including the upper bound, a fundamental lower bound, and a counterexample showing the failure of ignoring imperfection. Additional results and proofs are provided in the appendices.

% ==============================================================================================================================================================================
\section{Related Work}\label{sec:relatedwork}

We organize related work into six themes spanning RLHF theory, online exploration, robust learning, and preference heterogeneity. Our key distinction is a regret-oriented treatment of \emph{multi-source} preference feedback that can \emph{deviate} cumulatively and adaptively from an ideal oracle, quantified by a budget $\omega$.

\textbf{Preference-based RL and RLHF theory.} Preference-based RL has been studied from early trajectory-preference learning and reward modeling from comparisons \citep{jain2013learning,christiano2017deep} to regret and sample-efficiency guarantees with once-per-episode comparisons and general function approximation \citep{chatterji2021theory,chen2022human,zhu2023principled}. Recent theory further clarifies the statistical and computational structure of RLHF, including separations between RLHF and standard RL in episodic settings \citep{wang2023rlhf}, exploration-centric analyses and improved data utilization \citep{du2024exploration,xie2024exploratory}, and preference-oracle or KL-regularized formulations \citep{munos2023nash,swamy2024minimaximalist}. A common assumption across these works is that feedback is generated consistently from a single ground-truth oracle.

\textbf{Robustness to corrupted or noisy preference labels.} A growing literature studies robustness when preference labels are noisy or corrupted under structured noise models. R$^3$M treats corruption as sparse outliers in comparisons and learns a robust reward model via an $\ell_1$-regularized likelihood, with guarantees when the number of outliers is sublinear \citep{bukharin2024r3m}. Corruption-robust offline RLHF considers an $\varepsilon$-fraction of corrupted comparisons in an offline dataset and derives pessimistic/robust policy guarantees under coverage assumptions \citep{mandal2024corruption}. In LLM post-training, robust preference optimization methods (e.g., ROPO) filter or reweight suspected noisy comparisons and provide supporting theory \citep{liang2024ropo}, while rDPO analyzes robustness to preference flips and proposes a de-biased objective \citep{chowdhury2024rdpo}. These works primarily focus on offline or post-training objectives and assume specific corruption structures, e.g., sparse outliers, $\varepsilon$-fraction contamination, or random/flip noise. In contrast, we study multi-source imperfect feedback in an online setting and characterize regret as a function of $(K,M,\omega)$ under a cumulative deviation budget in preference probabilities

\textbf{Reward/preference model misspecification and model uncertainty.} A different robustness axis concerns misspecification of the preference link model (e.g., Bradley--Terry--Luce or Thurstone) and/or misspecification and uncertainty in the learned reward/preference model used for policy optimization. VRPO proposes a variance-reduced RLHF pipeline under reward/preference model misspecification, with improved regret-style guarantees and empirical gains in LLM fine-tuning~\citep{ye2025vrpo}. Reward-robust RLHF explicitly accounts for reward-model uncertainty and designs objectives that are stable under reward perturbations or model errors in LLM post-training~\citep{yan2024reward}. Uncertainty-aware preference alignment models stochasticity in preferences via distributional reward learning and risk-sensitive optimization (e.g., CVaR) on offline preference datasets~\citep{xu2024uncertainty}. These works primarily target robustness to modeling error (offline and/or LLM-centric).

\textbf{Distributionally robust RLHF.} A distinct robustness axis concerns input/task distribution shift between fine-tuning (or data collection) and deployment, e.g., changes in the prompt distribution or evaluation tasks. Distributionally robust RLHF formulates DRO variants of reward-based RLHF and reward-free DPO to improve out-of-distribution performance under shifted prompt distributions, with supporting theoretical guarantees for the resulting objectives~\citep{mandal2025dro}. In contrast, we study feedback-side (within-task) multi-source preference imperfection across episodes and characterize regret in terms of the explicit tradeoff between statistical averaging $\sqrt{K/M}$ and cumulative mismatch $\omega$.

\textbf{Heterogeneous preferences and equitable/personalized alignment.} Several works explicitly model preference heterogeneity across annotators or users and study equitable or personalized alignment, e.g., max--min/group-robust objectives~\citep{chakraborty2024maxmin} and personalization or preference aggregation under heterogeneous feedback~\citep{park2024rlhf}. These efforts treat heterogeneity as multiple valid objectives to be satisfied simultaneously. In contrast, we evaluate against a single ideal preference oracle (induced by a ground-truth objective) and treat cross-source discrepancies as imperfect feedback relative to this oracle, quantified by a cumulative budget $\omega$.

\textbf{Coresets for rich function approximation.} Subsampling and coreset constructions are classical tools for controlling sensitivity and statistical complexity in regression and learning~\citep{drineas2006sampling,langberg2010universal}. Recent RL work leverages online subsampling together with eluder/covering-type arguments to obtain provably efficient learning with rich function approximation~\citep{russo2013eluder,kong2021online}. Our Step~4 tailors these ideas to comparison learning under imperfect feedback, where history-adaptive weights make objectives sensitive to a subset of trajectories.

Surveys and critical analyses systematize the RLHF pipeline and its failure modes, including reward-model misspecification, sparse or biased feedback, and misgeneralization~\citep{kaufmann2023survey,casper2023open}. These works motivate the need for principled robustness guarantees. Complementary to these perspectives, our work provides a regret-oriented characterization for multi-source imperfect preference feedback under a cumulative deviation budget $\omega$, including an information-theoretic lower bound, a regret upper bound with the same scaling in $(K,M,\omega)$, and a best-of-both-regimes guarantee.

% ==============================================================================================================================================================================
\section{Problem Formulation}\label{sec:problemformulation}

Notation: Let $[K]\triangleq\{1,2,\ldots,K\}$ and $[H]\triangleq\{1,2,\ldots,H\}$. For any scalar $x$, define $[x]_{\ge 1}\triangleq \max\{x,1\}$. In episode $k$, let $\tau_k \triangleq (s_{k,1},a_{k,1},\ldots,s_{k,H},a_{k,H})$ denote the trajectory generated by the executed policy, and let $\tau_{k,0}$ denote the reference trajectory generated by a fixed reference policy $\pi_0$. Let $\Gamma \triangleq (\gS\times\gA)^H$ be the trajectory space. Let $\Gamma_k \triangleq \{\tau_t\}_{t=1}^{k-1}$ denote trajectories generated by the learning algorithm before episode $k$. Let $\mathcal{D}_k \triangleq \{(\tau_t,\tau_{t,0})\}_{t=1}^{k-1}$ denote the set of comparison pairs collected before episode $k$. For each step $h\in[H]$, let $\Gamma_{k,h} \triangleq \{(s_{t,h},a_{t,h})\}_{t=1}^{k-1}$ denote the set of stage-$h$ state--action pairs collected before episode $k$, and let $\gT_{k,h} \triangleq \{(s_{t,h},a_{t,h},s_{t,h+1})\}_{t=1}^{k-1}$ denote the set of step-$h$ transition tuples collected before episode $k$. We use $\tilde{O}(\cdot)$ to hide constant and polylogarithmic factors in $(K,H,d,1/\delta)$.

We study online reinforcement learning (RL) in an episodic Markov decision process (MDP) $\gM=(H,\gS,\gA,P^*,R^*)$ over $K$ episodes. The environment dynamics are Markov with an unknown transition kernel $P^*=\{P_h^*\}_{h=1}^H$, where $P_h^*(\cdot\mid s,a)$ denotes the distribution of $s_{h+1}$ given $(s_h,a_h)$. The learning objective is induced by an unknown trajectory-based reward functional $R^*:\Gamma\to\sR$, which may be non-Markovian. We assume $R^*$ is bounded for all $\tau\in\Gamma$. The initial state is fixed: $s_{k,1}=s_1$ for all $k$. Accordingly, we allow history-dependent (possibly randomized) policies: a policy $\pi\in\Pi$ maps each history $(s_{1},a_{1},\ldots,s_{h})$ to a distribution over actions in $\gA$, and executing $\pi$ in $P^*$ induces a trajectory random variable $\tau^\pi\sim(P^*,\pi)$. At the start of episode $k$, the learner selects $\pi_k$ as a (possibly randomized) function of the past history $\gH_{k-1}\triangleq\{(\tau_t,\tau_{t,0}), f_t^{1:M}\}_{t=1}^{k-1}$, where $f_t^{1:M}\triangleq(f_t^1,\ldots,f_t^M)$ are the $M$ preference labels observed in episode $t$ (defined below).

\textbf{Preference feedback:} Fix a known reference policy $\pi_0$, e.g., a safe/approved policy, a pretrained model, or a heuristic controller. In each episode $k$, the learner rolls out $\pi_k$ to obtain $\tau_k\sim(P^*,\pi_k)$, and independently rolls out $\pi_0$ to obtain $\tau_{k,0}\sim(P^*,\pi_0)$. An \emph{ideal} oracle returns a binary preference label $f_k^*\in\{0,1\}$, where $f_k^*=1$ indicates $\tau_k\succ\tau_{k,0}$. The preference probability is generated through a known link function $\smsigma:\sR\to[0,1]$, as follows,
\begin{align}
p_k^* \triangleq \Pr(f_k^*=1\mid \tau_k,\tau_{k,0}) = \smsigma \left(R^*(\tau_k)-R^*(\tau_{k,0})\right).
\label{eq:def-ideal-pref}
\end{align}
Typical choices include Bradley--Terry--Luce and Thurstone models. We assume $\smsigma$ is nondecreasing, $L_\sigma$-Lipschitz ($|\smsigma(x)-\smsigma(y)|\le L_\sigma|x-y|$ for all $x,y$) and symmetric.

\begin{remark}[Identifiability and learnable object]
Preference feedback depends on $R$ only through differences $R(\tau)-R(\tilde{\tau})$. Thus, $R^*$ is identifiable only up to an additive constant, while the induced \emph{comparison function} $q_R(\tau,\tilde{\tau}) \triangleq \smsigma \left(R(\tau)-R(\tilde{\tau})\right)$ is invariant to additive shifts of $R$. Our algorithm estimates a reward model $\hat R$ as a convenient parameterization, but all guarantees are stated in terms of the induced comparison probabilities $q_R$, which are the quantities directly connected to the feedback and the performance metric.
\end{remark}

% =========================================================
\subsection{Multi-Source Imperfect Preference Feedback}

In practice, feedback may deviate from the ideal oracle due to annotator heterogeneity, expertise variation, or reward-model/labeling artifacts. We model this by $M$ feedback sources indexed by $m\in[M]$. In episode $k$, after observing the comparison pair $(\tau_k,\tau_{k,0})$, the learner receives one binary label from each source, $\{f_k^m\}_{m=1}^M$ with $f_k^m\in\{0,1\}$. Each source behaves as if it evaluates trajectories using an unknown perturbed reward functional $R^m$, inducing the preference probability
\begin{align}
p_k^m \triangleq \Pr(f_k^m=1\mid \tau_k,\tau_{k,0}) = \smsigma \left(R^m(\tau_k)-R^m(\tau_{k,0})\right) = q_{R^m}(\tau_k,\tau_{k,0}). \label{eq:def-source-pref}
\end{align}
Conditioned on $(\tau_k,\tau_{k,0})$, we assume $\{f_k^m\}_{m=1}^M$ are independent Bernoulli random variables with means $\{p_k^m\}_{m=1}^M$. Define the per-episode deviation (imperfection) of source $m$ in episode $k$ as $\Delta_k^m \triangleq p_k^m-p_k^*$. Rather than imposing a parametric model on $\{R^m\}_{m=1}^M$, we assume a \emph{cumulative imperfection budget}: there exists a common $\omega\ge 0$ such that for any (possibly adaptive) sequence of executed policies $\{\pi_k\}_{k=1}^K$ (equivalently, for any adaptively generated sequence of comparison pairs $\{(\tau_k,\tau_{k,0})\}_{k=1}^K$), and for every source $m\in[M]$, we have
\begin{align}
\sum\nolimits_{k=1}^K |\Delta_k^m| = \sum\nolimits_{k=1}^K \left|p_k^m-p_k^*\right| \le \omega. \label{eq:def-uncertainty-budget}
\end{align}
This allows heterogeneous, history-dependent deviations but bounds their total magnitude.

\begin{example}[LLM RLHF: annotator/reward-model mismatch]\label{example:llm}
In LLM alignment, feedback may come from multiple annotators and/or learned reward models. Different sources can disagree on what constitutes ``helpful'' versus ``harmless,'' appropriate verbosity, or stylistic preferences, inducing source-dependent preference probabilities $p_k^m$ that deviate from an idealized oracle $p_k^*$ across episodes.
\end{example}

\begin{example}[Autonomous driving: heterogeneous criteria]\label{example:driving}
In autonomous driving, feedback sources (users, safety engineers, domain experts, learned critics) may emphasize different criteria such as safety, comfort, efficiency, courtesy, and timeliness. Such differences induce source-dependent preferences that can deviate from a single holistic objective, naturally aligning with a
cumulative deviation budget.
\end{example}

% =========================================================
\subsection{Performance Metric}

We evaluate an online algorithm $\alg$ by pseudo-regret measured against the ground-truth preference utility induced by the ideal oracle. For any policy $\pi\in\Pi$, define
\begin{align}
L^\pi \triangleq \E_{\substack{\tau^\pi\sim(P^*,\pi) \\ \tau^{\pi_0}\sim(P^*,\pi_0)}} \left[q_{R^*} \left(\tau^\pi,\tau^{\pi_0}\right)\right] = \E_{\substack{\tau^\pi\sim(P^*,\pi) \\ \tau^{\pi_0}\sim(P^*,\pi_0)}} \left[\smsigma \left(R^*(\tau^\pi)-R^*(\tau^{\pi_0})\right)\right],
\label{eq:def-utility}
\end{align}
where $\tau^\pi$ and $\tau^{\pi_0}$ are independent rollouts under dynamics $P^*$. Let $L^* \triangleq \sup_{\pi\in\Pi} L^\pi$. Let $\pi_k$ denote the policy executed by $\alg$ in episode $k$. The regret after $K$ episodes is
\begin{align}
\regret^{\alg}(K) \triangleq \sum\nolimits_{k=1}^K \big( L^* - L^{\pi_k} \big).
\label{eq:def-regret}
\end{align}
The learner observes $\{f_k^m\}_{m=1}^M$ generated by imperfect sources, while performance is measured with respect to $(P^*,R^*)$ (or equivalently $q_{R^*}$). Our goal is to design algorithms whose regret guarantees hold uniformly over all multi-source feedback processes satisfying the imperfection budget (\ref{eq:def-uncertainty-budget}), and to characterize how $\regret^{\alg}(K)$ scales with $(K,M,\omega)$. %To the best of our knowledge, this is the first work establishing regret guarantees for episodic RLHF with multi-source preference feedback that can deviate from a shared ground-truth preference oracle under a cumulative imperfection budget.

% ==============================================================================================================================================================================
\section{Algorithm Design}\label{sec:algorithmdesign}

We design an online algorithm (see Algorithm~\ref{alg:mainnewalgorithm}) for RL from multi-source imperfect preference feedback in Section~\ref{sec:problemformulation}, aiming for a unified best-of-both-regimes guarantee in terms of $(K,M,\omega)$. For clarity, we describe the method using optimization oracles. In the linear setting, the resulting subproblems are convex and can be solved efficiently, while under general function approximation, one may use approximate oracles.

The imperfect-feedback model creates a fundamental tension between robustness and statistical efficiency: (i) \emph{High-imperfection regime.} When $\omega$ is large, na\"ively pooling historical labels can accumulate systematic bias, and may lead to linear regret. (ii) \emph{Low-imperfection regime.} When $\omega$ is small, aggregating $M$ labels per episode intuitively sharpens comparison estimation, so regret should improve with $M$. (iii) \emph{Unified algorithm across regimes.} We seek a unified method whose guarantee interpolates between these behaviors, i.e., a best-of-both-regimes guarantee: it should retain the $M$-dependent statistical gains in the low-imperfection regime while remaining robust in the high-imperfection regime. Technically, this requires controlling the cumulative mismatch so it contributes only an additive $\omega$ term, while still exploiting multi-source averaging.

Algorithm~\ref{alg:mainnewalgorithm} addresses these challenges with four components: imperfection-adaptive weighted aggregation to learn the comparison function while preventing bias accumulation; hidden-effect-targeted, value-weighted transition regression to control feedback-induced data dependence in model learning; policy-level optimism tailored to preference-only feedback for exploration; and sub-importance sampling to control sensitivity so the weighted objectives remain stable and analyzable. The detailed weight/normalization schedules and confidence-radius choices are deferred to Appendix~\ref{subsec:detailedsteps}.

\begin{algorithm}[t]
\caption{\mainalgoname~(\mainalgoabrv)}
\label{alg:mainnewalgorithm}
\begin{algorithmic}[1]
\STATE \textbf{Initialization:} Choose regularizers $\alphar,\alphap$ and confidence schedules $\{\betar_k,\betap_k\}_{k=1}^K$. Raw histories $\Gamma_{1}^{\rm raw}\gets\emptyset$ and $\hat\gT_{1,h}^{\rm raw}\gets\emptyset$ for all $h\in[H]$. Filtered histories $\Gammaa_{1}\gets\emptyset$ and $\hat\Gamma_{1,h}\gets\emptyset$ for all $h\in[H]$. Initialize the weight-update rules for $\{\weightrt\}$ and $\{\weightpt\}$.
\FOR{episode $k=1:K$}
    \STATE \textbf{Step 1:} Using the filtered comparison history $\hat{\gD}_k \triangleq \{(\tau,\tau_{t(\tau),0}):\tau\in\Gammaa_k\}$, compute $\hat R_k$, confidence set $\gQ_k$, and comparison bonus $\bonr_k(\cdot)$ via (\ref{eq:mainalgorithmestimatesteinerrewardcenter})--(\ref{eq:mainalgorithmbonusrewardfunction}).
    \STATE \textbf{Step 2:} Compute $\hatp_k=\{\hatp_{k,h}\}_{h=1}^H$, $\gP_k$, and the transition bonus $\bonp_k(\cdot)$ via (\ref{eq:mainalgorithmestimatehumanfeedbackPhat})--(\ref{eq:mainalgorithmbonustransitionkernel}).
    \STATE \textbf{Step 3:} Compute $\pi_k\in\argmax_{\pi\in\Pi}\UCB_k(\pi)$ via (\ref{eq:ucb-policy})--(\ref{eq:mainalgorithmupdateoplicy}).
    \STATE Execute policy $\pi_k$ and the reference policy $\pi_0$ to obtain trajectories $\tau_k$ and $\tau_{k,0}$, and observe feedback $\{f_k^m\}_{m=1}^{M}$.
    \STATE Update raw histories: $\Gamma_{k+1}^{\rm raw}\gets \Gamma_k^{\rm raw}\cup\{(\tau_k,\tau_{k,0})\}$, and for each $h\in[H]$, $\hat\gT_{k+1,h}^{\rm raw} \gets \hat\gT_{k,h}^{\rm raw}\cup\{(s_{k,h},a_{k,h},s_{k,h+1})\}$.
    \STATE \textbf{Step 4:} Form the filtered comparison history $\Gammaa_{k+1}$ from $\Gamma_{k+1}^{\rm raw}$, and the filtered transition histories $\{\hat\gT_{k+1,h}\}_{h=1}^H$ from $\{\hat\gT_{k+1,h}^{\rm raw}\}_{h=1}^H$, according to the sampling rule (\ref{eq:dfprsubimportancesampling-clean}).
\ENDFOR
\end{algorithmic}
\end{algorithm}

\textbf{Step 1: Comparison learning via imperfection-adaptive weighted aggregation.} A na\"ive approach pools all labels $\{f_t^m\}$ and fits a comparison model by uniformly minimizing a surrogate loss in the link space. Under the cumulative imperfection budget (\ref{eq:def-uncertainty-budget}), persistent source-dependent mismatches can accumulate into a significantly large estimation bias, which then propagates to planning and can lead to an $\omega\sqrt{K}$-type degradation.

We instead use imperfection-adaptive regression, based on truncation-of-inverse-bonus weighting in corruption-robust OFUL-style methods~\citep{he2022nearly}. Specifically, each episode $t$ is assigned a history-measurable weight $\weightrt\in(0,1]$ that acts as an uncertainty parameter. Intuitively, when imperfection is small the weights remain close to $1$ and we recover efficient multi-source averaging. When imperfection is large, low-confidence episodes are downweighted, which prevents cumulative mismatch from amplifying into $\omega\sqrt{K}$. Moreover, let $\hat{\mathcal D}_k$ be the filtered multiset of comparison pairs maintained by Step~4, i.e., $\hat{\mathcal D}_k \triangleq \{(\tau,\tau_{t(\tau),0}) : \tau\in \hat\Gamma_k\}$, where $t(\tau)$ denotes the episode in which $\tau$ was collected. We compute a reward parameterization $\hat R_k$ (used through induced comparisons) by solving
\begin{align}
\hat{R}_k \triangleq \arginf\nolimits_{ R \in \gR } \sum\nolimits_{(\tau,\tau_0)\in \hat{\gD}_{k}} \sum\nolimits_{m=1}^{M} \weightrt ( q_R(\tau,\tau_0) - f_{t(\tau)}^{m} )^2, \label{eq:mainalgorithmestimatesteinerrewardcenter}
\end{align}
where weights $\weightrt$ uses the episode index $t(\tau)$ and are computed via self-normalization on prefix data (see Appendix~\ref{subsec:detailedsteps}). We then define a comparison confidence set (in the learnable comparison space) by controlling deviations on the same filtered pairs:
\begin{align}
\gQ_k \triangleq \{ q_R: R\in\gR, \sum\nolimits_{(\tau,\tau_0)\in\hat{\gD}_k} \weightrt \big[q_R(\tau,\tau_0)-q_{\hat R_k}(\tau,\tau_0)\big]^2 \le \betar_k \}, \label{eq:mainalgorithmestimatesteinerconfidenceset}
\end{align}
where $\betar_k$ is chosen so that $q_{R^*}$ lies in $\gQ_k$ with high probability. Finally, we define the comparison exploration bonus on a candidate comparison pair $(\tau,\tau_0)$ as
\begin{align}
\bonr_k(\tau,\tau_0) \triangleq \sup\nolimits_{q\in \gQ_k} \weightrrk \left(q(\tau,\tau_0)-q_{\hat R_k}(\tau,\tau_0)\right). \label{eq:mainalgorithmbonusrewardfunction}
\end{align}
%In Step~3 we instantiate (\ref{eq:mainalgorithmbonusrewardfunction}) using $(\tau,\tau_0)=(\tau,\tau_{k,0})$ for the episode-$k$ comparison (and $\bonr_k(\tau_{k,0},\tau_{k,0})\equiv 0$).

\textbf{Step 2: Transition estimation via hidden-effect-targeted, value-weighted regression.} Although the transition kernel $P^*=\{P_h^*\}_{h=1}^H$ is not directly corrupted by preference labels, the \emph{distribution} of transition samples is induced by the executed policies $\{\pi_t\}$, which are updated using multi-source imperfect feedback. Hence, preference imperfection can influence transition learning through a hidden feedback to policy and then to data pathway. Our transition estimator is designed to be compatible with the same filtered history and imperfection-adaptive weighting used in Step~1, which is essential in the analysis for preventing this pathway from amplifying a budget $\omega$ into a looser $\omega\sqrt{K}$-type term.

Rather than estimating $P^*$ in a distribution-agnostic metric, we fit $P^*$ in \emph{planning-relevant directions} via value-targeted moments (cf. value-targeted regression in \cite{ayoub2020model}). Let $\hat{\gT}_{k,h}$ be the filtered multiset of step-$h$ transition tuples $\zeta=(s,a,s')$ collected before episode $k$ (maintained by Step~4), and let $t(\zeta)$ denote the episode in which $\zeta$ was collected. For each tuple $\zeta\in\hat{\gT}_{k,h}$, let $V_{t(\zeta),h+1}\in\gV_{h+1}$ be a history-measurable probe/value function chosen \emph{before} observing $s'$ (e.g., the maximizer of the transition uncertainty at $(s,a)$ under the then-current confidence set, as in value-targeted regression). Since $s'\sim P_h^*(\cdot\mid s,a)$ and $V_{t(\zeta),h+1}$ is fixed given the history, $V_{t(\zeta),h+1}(s')$ is an unbiased sample of $\langle P_h^*(\cdot\mid s,a),V_{t(\zeta),h+1}\rangle$ conditional on $(s,a)$. Then, we compute the transition confidence center $\hatp_k=\{\hatp_{k,h}\}_{h=1}^H$ by the weighted least-squares problem
\begin{align}
\hatp_k \triangleq \argmin_{P\in\gP} \sum\nolimits_{h=1}^{H} \sum\nolimits_{\zeta=(s,a,s')\in \hat{\gT}_{k,h}} w^{\text{P}}_{3,t(\zeta),h} \Big( \langle P_h(\cdot\mid s,a), V_{t(\zeta),h+1}\rangle - V_{t(\zeta),h+1}(s') \Big)^2, \label{eq:mainalgorithmestimatehumanfeedbackPhat}
\end{align}
where $w^{\text{P}}_{t,h}$ (denoted $\weightpt$ in the pseudocode) is a history-measurable inverse-uncertainty weight, chosen analogously to Step~1. Based on $\hatp_k$, we construct a transition confidence set by controlling the same value-targeted prediction error:
\begin{align}
\gP_k \triangleq \Big\{ P\in\gP: \sum_{h=1}^{H} \sum_{ \substack{ \zeta=(s,a,s') \\ \in \hat{\gT}_{k,h}}} w^{\text{P}}_{3,t(\zeta),h} \Big[ \big\langle \big(P_h(\cdot\mid s,a)-\hatp_{k,h}(\cdot\mid s,a)\big), V_{t(\zeta),h+1}\big\rangle \Big]^2 \le \betap_k \Big\}, \label{eq:mainalgorithmestimatetransitionconfidenceset}
\end{align}
where $\betap_k$ is chosen so that $P^*\in\gP_k$ with high probability. Finally, we define a planning-relevant width (bonus) along a trajectory $\tau=(s_1,a_1,\ldots,s_H,a_H)$ by
\begin{align}
\bonp_k(\tau) \triangleq \sum\nolimits_{h=1}^{H} \sup\nolimits_{P\in\gP_k} \sup\nolimits_{V\in\gV_{h+1}} \weightppk \Big| \big\langle \big(P_h(\cdot\mid s_h,a_h)-\hatp_{k,h}(\cdot\mid s_h,a_h)\big),\, V \big\rangle \Big|, \label{eq:mainalgorithmbonustransitionkernel}
\end{align}
%Equivalently one may use the diameter form $\max_{P_1,P_2\in\gP_k}\max_{V\in\gV_{h+1}}|\langle(P_{1,h}-P_{2,h})(\cdot\mid s_h,a_h),V\rangle|$, which differs by at most a factor $2$ for symmetric sets. This construction measures transition uncertainty only in directions that enter planning and, together with matched weighting/filtering, prevents feedback-driven distribution shift from destabilizing transition estimation.

\textbf{Step 3: Policy-level optimistic planning.} Because feedback is purely comparative and the objective is trajectory-level, standard Bellman-style optimism with a well-defined per-state reward is not directly applicable. Following the policy-level optimism principle in preference-based RL, we instead build an upper confidence bound on the \emph{reference-benchmarked} preference utility of each policy. Moreover, unlike settings that only observe preferences between two learner-chosen rollouts and therefore maintain a near-optimal policy set and then pick an exploratory policy pair~\citep{chen2022human}, our feedback always compares the learner's trajectory to a fixed reference policy $\pi_0$. This benchmarked structure allows us to score and optimize a \emph{single} policy via an optimistic surrogate.

Define the plug-in estimate of the benchmarked utility under the estimated model:
\begin{align}
\widehat{L}_k(\pi) \triangleq \E_{\tau\sim(\hatp_k,\pi), \tau_0\sim(\hatp_k,\pi_0)} \big[q_{\hat R_k}(\tau,\tau_0)\big],
\end{align}
where $\tau$ and $\tau_0$ are generated independently under $\hatp_k$ by $\pi$ and $\pi_0$, respectively. We then define the policy-level UCB as
\begin{align}
\UCB_k(\pi) \triangleq \widehat{L}_k(\pi) + \E_{\tau\sim(\hatp_k,\pi), \tau_0\sim(\hatp_k,\pi_0)} \Big[ \bonr_k(\tau,\tau_0) +\bonp_k(\tau) +\bonp_k(\tau_0) \Big], \label{eq:ucb-policy}
\end{align}
where $\bonr_k(\tau,\tau_0)$ upper-bounds the comparison-function uncertainty (Step~1) and $\bonp_k(\tau),\bonp_k(\tau_0)$ account for transition uncertainty along the executed and reference rollouts (Step~2). We choose the updated policy by
\begin{align}
\pi_k \in \argmax\nolimits_{\pi\in\Pi} \UCB_k(\pi).
\label{eq:mainalgorithmupdateoplicy}
\end{align}
Intuitively, when imperfection is small, aggregating $M$ labels tightens the comparison confidence set and shrinks $\bonr_k$ at an $M$-dependent statistical rate, while when imperfection is large, the self-normalized weighting in Steps~1--2 prevents the cumulative budget from inflating the optimistic bonuses multiplicatively across episodes. Thus, policy-level optimism coupled with these bonuses yields a single update rule that interpolates smoothly between the low- and high-imperfection regimes.

\textbf{Step 4: Sub-importance sampling (filtered data) for controlling sensitivity.} Steps~1--2 rely on history-adaptive inverse-uncertainty weights and planning-relevant regression objectives. As observed in stability analyses of bonus-based RL, if all past data are retained, then the induced bonus/confidence construction can have very high complexity and be overly sensitive to a small subset of samples, which complicates concentration and can lead to loose regret bounds~\citep{wang2020reinforcement}. Our strategy is therefore to maintain \emph{filtered multisets} whose (weighted) geometry approximately preserves the confidence regions while keeping the effective sample size small.

Specifically, for a multiset $\gZ$ of samples $z$ and nonnegative weights $w(z)$, define
\begin{align}
\|f-g\|_{\gZ,w}^2 \triangleq \sum\nolimits_{z\in\gZ} w(z)\,\big(f(z)-g(z)\big)^2.
\end{align}
We apply the sampling procedure to comparison learning with samples $z=(\tau,\tau_{t(\tau),0})$ and weights $w(z)=\weightrt$, and to transition learning with samples $z=(s,a,s')$ at each stage $h$ and weights $w(z)=\weightpt$. Given a function class $\Fcal$ (the set of prediction functionals that appear in Steps~1--2 and in the bonus construction), a threshold $\lambda>0$, and a weighted dataset $(\mathcal{Z},w)$, define the sub-importance of $z\in\mathcal{Z}$ by
\begin{align}
\Imp_{\gZ,\Fcal,\lambda,w}(z) \triangleq \sup\nolimits_{f,g\in\Fcal: \|f-g\|_{\mathcal{Z},w}^2\ge \lambda} w(z) \big(f(z)-g(z)\big)^2 / \|f-g\|_{\mathcal{Z},w}^2. \label{eq:def-weighted-sensitivity}
\end{align}
At the end of episode $k$, let $\mathcal{Z}_k^{\mathrm{cmp}}$ be the multiset of all past comparison pairs $\{(\tau_t,\tau_{t,0})\}_{t<k}$ and $\mathcal{Z}_{k,h}^{\mathrm{tr}}$ be the multiset of all past step-$h$ transition tuples $\{(s_{t,h},a_{t,h},s_{t,h+1})\}_{t<k}$. For each sample $z$ in the relevant dataset, set
\begin{align}
p_k(z) \triangleq \min\Big\{1, c \cdot \Imp_{\mathcal{Z},\Fcal,\lambda,w}(z)/\epsilon^2\Big\}, \label{eq:dfprsubimportancesampling-clean}
\end{align}
where $\epsilon\in(0,1)$ is an accuracy parameter and $c>0$ hides logarithmic factors. Independently for each $z$, we include $z$ in the filtered multiset with multiplicity $1/p_k(z)$ with probability $p_k(z)$ (so the expected multiplicity is one), and discard it otherwise. This produces the filtered multisets $\hat\Gamma_k$ (for comparisons) and $\{\hat{\gT}_{k,h}\}_{h=1}^H$ (for transitions).

Intuitively, by construction, the filtered objectives are unbiased estimators of the full-history weighted objectives, while importance sampling proportional to sensitivity controls the variance and yields (with high probability) an $\epsilon$-approximation guarantee for all $f,g\in\Fcal$: $\|f-g\|_{\hat{\mathcal{Z}},w}^2 \approx \|f-g\|_{\mathcal{Z},w}^2$. Consequently, the confidence regions (and hence bonuses) built from filtered data are approximately preserved, while the number of \emph{distinct} retained samples is controlled by the eluder/covering complexity of $\Fcal$.

% ==============================================================================================================================================================================
\section{Theoretical Results}\label{sec:theoreticalanalysis}

This section provides theoretical results for episodic RLHF with multi-source imperfect preference feedback in Section~\ref{sec:problemformulation}. We establish an information-theoretic regret lower bound under the cumulative imperfection budget (\ref{eq:def-uncertainty-budget}), a counterexample showing that na\"ively ignoring imperfection can be substantially worse, and a regret upper bound for Algorithm~\ref{alg:mainnewalgorithm}. The key parameters are the number of episodes $K$, the number of feedback sources $M$, and the cumulative imperfection budget $\omega$. We hide polylogarithmic factors in $\tilde{O}(\cdot)$ and $\tilde{\Omega}(\cdot)$.

% ==========================================================
\subsection{A Lower Bound}\label{subsec:regretlowerbound}

We first characterize the fundamental hardness of learning under an imperfection budget. Recall that in episode $k$ the ideal preference probability is $p_k^* = q_{R^*}(\tau_k,\tau_{k,0})$, while source $m$ induces $p_k^m = q_{R^m}(\tau_k,\tau_{k,0})$, and for every $m\in[M]$ the deviations satisfy $\sum_{k=1}^K |p_k^m-p_k^*|\le \omega$.

\begin{theorem}[Lower bound]\label{thm:lowerboundomega}
Under the setting in Section~\ref{sec:problemformulation}, for any (possibly randomized) algorithm $\alg$ there exists an instance satisfying (\ref{eq:def-uncertainty-budget}), such that
\begin{align}
\mathbb{E} [ \regret^{\alg}(K) ] \ge \tilde{\Omega} ( \max \{ \sqrt{K/M}, \omega \} ), \label{eqthm:lowerboundomega}
\end{align}
where the expectation is over the randomness of the environment, algorithm, and feedback.
\end{theorem}

Please see Appendix~\ref{appendix:pfthmlowerboundxi} for a complete form of Theorem~\ref{thm:lowerboundomega} and the proof. The $\sqrt{K/M}$ term captures the $1/\sqrt{M}$ variance reduction from observing $M$ conditionally independent binary labels per episode, while the $\omega$ term shows that \emph{additive} dependence on cumulative imperfection is information-theoretically unavoidable. Moreover, we do not assume $\omega=o(K)$ in the model. Instead, regret is characterized as a function of $(K,M,\omega)$. Since $\regret^{\alg}(K)\le K$ always, the lower bound should be read as $\tilde{\Omega} \left( \min\{ K,\max\{ \sqrt{K/M},\omega \} \} \right)$. In particular, if $\omega=\Theta(K)$, then there exist admissible feedback processes (satisfying the budget) under which any algorithm suffers $\tilde{\Omega}(K)$ regret, so no method can guarantee sublinear regret uniformly over all processes satisfying (\ref{eq:def-uncertainty-budget}). Thus, achieving sublinear regret requires $\omega=o(K)$, i.e., a diminishing average mismatch $\bar\omega \triangleq \omega/K = o(1)$.

% =====================================================================================
\subsection{A Counterexample: Ignoring Imperfection Can Be Much Worse}\label{subsec:improvement}

We next show that a natural ``ignore-imperfection'' strategy can be provably suboptimal. Concretely, consider the following baseline: it pools all past labels from all sources and fits a comparison predictor by traditional least squares in the comparison/link space (equivalently, it uses the per-episode average label $\bar f_k \triangleq \frac1M\sum_{m=1}^M f_k^m$), and then performs optimism using a confidence radius calibrated only for stochastic noise, e.g., Bernoulli/sub-Gaussian fluctuations. In other words, it treats the observed labels as unbiased samples of $q_{R^*}$ and uses OFUL-type principle without incorporating $\omega$ into the confidence construction.

\begin{proposition}[A counterexample]\label{prop:ignore-omega}
There exists an instance satisfying (\ref{eq:def-uncertainty-budget}), such that the above unweighted optimistic baseline suffers
\begin{align}
\mathbb{E} [ \regret^{\text{OFUL}}(K) ] \ge \tilde{\Omega} ( \min\{ \omega\sqrt{K},K \} ). \label{eqprop:ignore-omega}
\end{align}
\end{proposition}

Proposition~\ref{prop:ignore-omega} highlights the core failure mode. Without imperfection-adaptive weighting, a cumulative deviation budget $\omega$ can amplify into an $\omega\sqrt{K}$-type degradation, rather than contributing additively. Intuitively, unweighted regression absorbs the (possibly history-dependent) bias into the fitted comparison predictor, while standard OFUL-style confidence sets, which account only for stochastic noise, need not contain the true comparison function $q_{R^*}$. Optimism can then exploit the biased direction for many episodes, producing regret that scales as $\omega\sqrt{K}$ (capped by $K$). Notably, having multiple sources helps variance (the $\sqrt{K/M}$ term) but does not, by itself, prevent worst-case \emph{bias} accumulation, which is why imperfection-adaptive weighting is essential for the best-of-both-regimes guarantee.

In regimes such as $\omega = \tilde{\Theta}(K^{1/4})$, this bound in (\ref{eqprop:ignore-omega}) is strictly worse than both the fundamental rate $\tilde{\Omega}(\max\{\sqrt{K/M},\omega\})$ in the lower bound (\ref{eqthm:lowerboundomega}) and our algorithm's regret upper bound $\tilde{O}(\sqrt{K/M}+\omega)$. See Appendix~\ref{appendix:proofofthmregretlowerbndexistingrlhf} for the detailed construction and proof.

% ==========================================================
\subsection{Regret Upper Bounds and Best-of-Both-Regimes Guarantee}\label{subsec:upperbound}

Now we present regret guarantees for Algorithm~\ref{alg:mainnewalgorithm} under multi-source imperfect preferences with cumulative budget $\omega$. To make the best-of-both-regimes phenomenon most transparent, we first state a known-transition special case, where the regret cleanly decomposes into a statistical term that improves as $1/\sqrt{M}$ and an unavoidable additive imperfection term proportional to $\omega$. We then return to the main setting with unknown transitions, where additional transition-learning terms appear. Finally, we state the extension to general function approximation in terms of general covering numbers and eluder dimensions.

We define the comparison class on $\gX_T\triangleq \Gamma\times\Gamma$ by $\gF_T \triangleq \{ f_R : (\tau,\tilde\tau)\mapsto q_R(\tau,\tilde\tau), R\in\gR \} \subseteq [0,1]^{\gX_T}$. Since the analysis and the bonuses involve differences between comparison predictors, we also use the associated difference class $\Delta\gF_T \triangleq \{ f-f' : f,f'\in\gF_T \}\subseteq [-1,1]^{\gX_T}$, and take the complexity measures with respect to $\Delta\gF_T$. For each $h\in[H]$, let $\gV_{h+1}\subseteq [0,H]^{\gS}$ be the class of bounded value functions used in value-targeted regression at step $h$. We define the stage-indexed domain $\gX_P \triangleq [H]\times\gS\times\gA$ and the value-targeted prediction class $\gF_P \triangleq \{ f_{P,V}:(h,s,a)\rightarrow \langle P_h(\cdot\mid s,a), V\rangle: P\in\gP, V\in\gV_{h+1} \} \subseteq [0,H]^{\gX_P}$, together with its difference class $\Delta\gF_P \triangleq \{f-f' : f,f'\in\gF_P\}\subseteq [-H,H]^{\gX_P}$.

% ==========================================================
\subsubsection{Upper Bounds Under Linear Function Approximation}\label{subsec:known-transition}

In this subsection, we first assume a linear mixture transition model~\citep{ayoub2020model}. Specifically, for each step \(h\in[H]\), there exists a known feature map \(\phi_P:\gS\times\gA\times\gS\to\mathbb R^{\tilde d_P}\) and a parameter \(\theta_{P,h}^*\in\Theta_P\) such that $P_h^*(s'\mid s,a) = \langle \phi_P(s,a,s'),\theta_{P,h}^* \rangle$, with $\|\sum_{s'} \phi_P(s,a,s') V(s') \|_2 \le 1$, $\forall (s,a)\in\gS\times\gA, V\in\gV$, and \(\|\theta_{P,h}^*\|_2\le B_P\) for all \(h\in[H]\). Moreover, there exists a feature map \(\phi_R:\Gamma\to\mathbb R^{\tilde d_T}\) and a parameter \(\theta_R^*\in\mathbb R^{\tilde d_T}\) such that $R^*(\tau) = \langle \phi_R(\tau),\theta_R^*\rangle$, with \(\|\phi_R(\tau)\|_2\le 1\) and \(\|\theta_R^*\|_2\le B_R\). The induced preference model is therefore $q_{R^*}(\tau,\tilde\tau) = \smsigma (\langle \phi_R(\tau)-\phi_R(\tilde\tau), \theta_R^*\big\rangle )$. See Appendix~\ref{app:linear-specialization} for the proof.

When the transition kernel \(P^*\) is known, Step~2 and the transition bonus \(\bonp_k\) are unnecessary. The learning problem reduces to comparison learning (Step~1) and optimistic policy selection (Step~3), so the \((K,M,\omega)\) tradeoff appears in its cleanest form.

\begin{theorem}[Known transition kernel]\label{thm:known-transition}
Fix \(\delta\in(0,1)\). Suppose the transition kernel \(P^*\) is known, and Algorithm~\ref{alg:mainnewalgorithm} is run without transition estimation and bonuses, i.e., \(\hat p_k\equiv P^*\) and \(\bonp_k\equiv 0\). Under multi-source imperfect preferences, with probability at least \(1-2\delta\),
\begin{align}
\regret^{\mainalgoabrv}(K) \le \tilde O (  \tilde d_T \sqrt{K / M} + \tilde d_T \omega ). \label{eq:known-transition-linear}
\end{align}
\end{theorem}

Theorem~\ref{thm:known-transition} exhibits the best-of-both-regimes tradeoff in its cleanest form. In the low-imperfection regime, when $\omega \lesssim \sqrt{K/M}$ up to \((\tilde d_T,\log)\) factors, the regret is dominated by the statistical term \(\tilde O(\tilde d_T \sqrt{K/M})\), which improves with the number of sources \(M\) through variance reduction. In the high-imperfection regime, when $\omega \gtrsim \sqrt{K/M}$, the regret is dominated by the imperfection term \(\tilde O(\tilde d_T\omega)\), where additional sources cannot help in the worst case because cumulative mismatch becomes the bottleneck. Thus, the crossover occurs around $\omega=\tilde\Theta(\sqrt{K/M})$, up to \((\tilde d_T,\log)\) factors.

We now return to the setting where the transition kernel \(P^*\) is unknown. Relative to the known-transition case
in Theorem~\ref{thm:known-transition}, additional regret is incurred because the learner must also estimate the
environment dynamics for planning.

\begin{theorem}[Unknown transition kernel]\label{thm:mainregrettheorem}
Fix \(\delta\in(0,1)\). Suppose the transition kernel \(P^*\) is unknown, and Algorithm~\ref{alg:mainnewalgorithm} is run with Step~2 transition estimation and transition bonuses \(\bonp_k\). Under multi-source imperfect preference feedback, with probability at least \(1-2\delta\),
\begin{align}
\regret^{\mainalgoabrv}(K) \le \tilde O ( \tilde d_T \sqrt{K/M} + \omega(\tilde d_T + H\tilde d_P) + \sqrt{H^2 \tilde d_P^3 K} ). \label{eq:main-upper-linear-unknownP}
\end{align}
\end{theorem}

Relative to Theorem~\ref{thm:known-transition}, the additional term \(\tilde O(\sqrt{H^2\tilde d_P^3K})\) is the price of learning unknown transitions for planning under the imperfect preference feedback. Importantly, this transition-learning term does not improve with the number of feedback sources \(M\) in our model, because the \(M\) sources provide multiple preference labels only for the same trajectory comparison pair \((\tau_k,\tau_{k,0})\) generated by a single environment interaction in episode \(k\). The transition samples \(\{(s_{k,h},a_{k,h},s_{k,h+1})\}_{h=1}^H\) are still observed only once per visited state-action pair per episode, independently of how many preference sources are queried.

The remaining terms retain the same best-of-both-regimes behavior in $(K,M,\omega)$. Multi-source averaging yields the $1/\sqrt{M}$ improvement when imperfection $\omega$ is low, while the cumulative imperfection budget contributes only additively when imperfection is high.

% =========================================================
\subsubsection{Upper Bounds Under General Function Approximation}\label{subsec:upperbound-general}

We next state the extension to general function approximation, where statistical terms are characterized by covering numbers and eluder dimensions adapted to preference feedback.

\begin{definition}[$\epsilon$-covering number]\label{definition:definitioncoveringnumber}
Let $(\gF,\|\cdot\|)$ be a metric space. A set $\{f_1,\ldots,f_N\}\subset\gF$ is an $\epsilon$-cover if for every $f\in\gF$, there exists $i\in[N]$ such that $\|f-f_i\|\le \epsilon$. The $\epsilon$-covering number $\gN(\gF,\|\cdot\|,\epsilon)$ is the minimal size of an $\epsilon$-cover.
\end{definition}

\begin{definition}[Eluder dimension]\label{definition:defineeluderdimension}
Let $\gF$ be a class of real-valued functions on $\gX$. For a sequence $\gX_N = \{x_1,\ldots,x_N\}\subseteq \gX$, a point $x\in\gX$ is $\epsilon$-dependent on $\gX_N$ with respect to $\gF$ if for all $f_1,f_2\in\gF$ satisfying $\sqrt{\sum_{n=1}^N (f_1(x_n)-f_2(x_n))^2}\le \epsilon$, we have $|f_1(x)-f_2(x)|\le \epsilon$; otherwise, $x$ is $\epsilon$-independent. The eluder dimension $\dim_E(\gF,\epsilon)$ is the largest integer $N$ for which there exist $\gX_N$ and $\epsilon' \ge \epsilon$, such that each $x_n$ is $\epsilon'$-independent of $\{x_1,\ldots,x_{n-1}\}$.
\end{definition}

\begin{theorem}[Unknown transitions with general function approximation]\label{thm:upperbound-general}
Fix $\delta\in(0,1)$. Suppose transition kernel $P^*$ is unknown, and Algorithm~\ref{alg:mainnewalgorithm} is run with parameters in (\ref{eq:beta-R-choice-exact}) and (\ref{eq:beta-P-choice-exact}). Under multi-source imperfect preferences, with probability at least $1-2\delta$,
\begin{align}
\regret^{\mainalgoabrv}(K) \le \tilde O \left( \sqrt{ (d_T K \log \frac{N_T}{\delta}) / M} + \omega(d_T+d_P) + \sqrt{d_P H K \log \frac{N_P}{\delta}} + HK\epsilon \right), \nonumber
\end{align}
where $\log (N_T/\delta) = \log (\gN(\Delta\gF_T,\|\cdot\|_\infty,1/K)/\delta)$, $\log (N_P/\delta) = \log (\gN(\Delta\gF_P,\|\cdot\|_\infty,1/K)/\delta)$, and $d_T = \dim_E(\Delta\gF_T,1/K)$, $d_P = \dim_E(\Delta\gF_P,1/K)$.
\end{theorem}

Theorem~\ref{thm:upperbound-general} preserves the same decomposition as in the linear case, but replaces the linear dimension by the complexity of the comparison and transition classes. Specifically, the first term is the statistical cost of learning the comparison model and improves by a factor \(1/\sqrt{M}\), because intuitively Step~1 can be analyzed using the averaged label from the \(M\) sources. The second term is the price of cumulative imperfection and remains additive, due to the self-normalized weighting adapted to imperfect preference. The third term is the cost of learning the unknown transition kernel and does not benefit from \(M\), because intuitively the \(M\) sources provide multiple preference labels for the same realized trajectories. The final term is the approximation error introduced by the filtered-history/sub-importance-sampling construction in Step~4. Please see Appendix~\ref{appendix:thm:upperbound-general} for the proof.

When \(P^*\) is known, the transition-learning term disappears and the bound reduces to the comparison-learning term plus the additive imperfection term, recovering the same best-of-both-regimes property as in the known-transition case. Likewise, in structured settings such as linear function approximation, the quantities \((d_T,\log N_T,d_P,\log N_P)\) specialize to finite-dimensional complexity parameters, recovering the linear-case theorem.

% ==========================================================
\subsection{Unknown Imperfection Budget}\label{subsec:unknown-omega}

The regret bounds above are instance-wise in the imperfection budget \(\omega\). However, the confidence radii and weight schedules in Steps~1--2 are most conveniently specified using an a priori upper bound on \(\omega\). We therefore consider the following plug-in variant. Run Algorithm~\ref{alg:mainnewalgorithm} exactly as before, but replace every occurrence of \(\omega\) in the confidence radii and weight schedules by a user-chosen upper bound \(\bar\omega\ge 0\).

\begin{theorem}[Unknown \(\omega\) via a plug-in upper bound \(\bar\omega\)]\label{thm:unknown-omega-plugin}
Fix \(\delta\in(0,1)\) and \(\bar\omega\ge 0\). Run Algorithm~\ref{alg:mainnewalgorithm} with \(\bar\omega\) substituted for \(\omega\). Then, with probability at least \(1-2\delta\), we have the same regret upper bound in Theorem~\ref{thm:upperbound-general}, except that $\omega$ is replaced by $\bar \omega$, if $\omega \le \bar\omega$. If \(\omega>\bar\omega\), then the confidence guarantees used in the proof may fail, and $\regret^{\mainalgoabrv}(K)\le K$.
\end{theorem}

Theorem~\ref{thm:unknown-omega-plugin} has the same two-case structure as the standard unknown-corruption guarantees in~\cite{he2022nearly}. A practically useful default is to set \(\bar\omega\) at the crossover scale of the low-imperfection regime. For example, in the linear model with known transitions, choosing \(\bar\omega=\sqrt{K/M}\) yields $\regret^{\mainalgoabrv}(K) = \tilde O (\tilde d_T\sqrt{K/M})$ whenever $\omega\le \sqrt{K/M}$. This gives a simple tuning-free guarantee throughout the full low-imperfection regime.

% ==============================================================================================================================================================================
\section{Conclusion and Future Work}\label{sec:conclusion}

We studied episodic RLHF with multi-source imperfect preference feedback under a cumulative imperfection budget. Unlike the standard single-oracle formulation, each feedback source may deviate from an ideal ground-truth comparison function, with cumulative deviation bounded by \(\omega\). Because comparison feedback identifies rewards only up to an additive constant, our analysis is carried out directly in the comparison space $q_R(\tau,\tilde\tau)=\smsigma\!\big(R(\tau)-R(\tilde\tau)\big)$. To handle this setting, we developed an algorithm that combines imperfection-adaptive weighted comparison learning, value-targeted transition estimation to control hidden feedback-induced distribution shift, policy-level optimism tailored to preference-only objectives, and sub-importance sampling to stabilize the weighted empirical objectives. Our regret guarantees explicitly characterize the joint roles of the number of episodes \(K\), the number of feedback sources \(M\), and the imperfection budget \(\omega\). In particular, the algorithm achieves a best-of-both-regimes behavior. When imperfection is small, aggregating multiple sources improves statistical efficiency through a \(1/\sqrt{M}\)-type gain, whereas when imperfection is large, the regret depends only additively on \(\omega\), which is unavoidable. Together with a lower bound, our results identify $\omega=\tilde{\Theta}(\sqrt{K/M})$ as the critical regime boundary beyond which additional sources provide diminishing returns.

Several directions naturally follow from this work. The first one is parameter-free adaptivity to unknown imperfection. Our current guarantees are instance-wise in \(\omega\), but the confidence radii and weight schedules are easiest to state using an upper bound on \(\omega\). Although plug-in bounds and doubling schemes can remove this requirement, it remains interesting to design fully adaptive, parameter-free methods that recover the same regret rates with minimal overhead and cleaner schedules. The second is heterogeneous users and multiple ground-truth objectives. In many RLHF applications, disagreement across sources does not merely reflect noise or imperfection, but genuine heterogeneity in preferences. Extending the framework to personalized or multi-population reward models, and studying robust notions of performance such as group-wise, worst-case, or distributionally robust preference utility, would substantially broaden the scope of the theory. The third is sharper dynamics-learning guarantees. In the unknown-transition setting, our analysis separates preference-learning and transition-learning effects. It would be valuable to further tighten the transition term under preference-driven exploration, and to study interaction models in which additional queried sources also provide additional environment information.

%\acks{}


\newpage

\appendix

% ============================================================================================================
\section{Detailed Design of the Weights}\label{subsec:detailedsteps}

This appendix specifies the weight/normalization schedules used in Steps~1--2 of Algorithm~\ref{alg:mainnewalgorithm}. The purpose of these weights is twofold. First is \emph{imperfection-adaptive downweighting}, so episodes with low inferred reliability contribute less to estimation, and second is \emph{self-normalized widths}, so the cumulative deviation budget $\omega$ enters the regret bound additively rather than being amplified into an $\omega\sqrt{K}$ term. Our construction follows the self-normalization principle used in corruption-robust contextual bandits and MDPs, e.g.,~\citep{he2022nearly,ye2023corruption}, and adapts it to preference comparisons and value-targeted transition regression.

%We use the standard convention that an inverse-uncertainty term $\frac{(\text{error})^2}{\hat\sigma^2}$ can be written as $w\cdot(\text{error})^2$ with $w \triangleq 1/\hat\sigma^2$, truncated so that $w\in(0,1]$ (equivalently, $\hat\sigma^2\ge 1$). All weights below have the generic form $[\text{uncertainty}]_{\ge 1}^{-1/2}$. Moreover, filtered multisets (e.g., $\hat{\mathcal D}_k$ and $\hat{\gT}_{k,h}$) may contain multiplicities. Whenever we index samples by $i=1,2,\ldots$, we assume the sequence is ordered by the collection episode $t(\cdot)$ (ties broken arbitrarily). Importantly, all uncertainty parameters and normalizers are computed using only prefix information available before the corresponding sample/episode, so the recursion is well-founded.

% =========================
\subsection{Weights for Comparison Learning (Step 1)}\label{app:weights-reward}

\textbf{Comparison path length (CPL).} Recall $q_R(\tau,\tilde\tau)=\smsigma(R(\tau)-R(\tilde\tau))$. For a multiset of comparison pairs $\mathcal{D}=\{(\tau_i,\tilde\tau_i)\}_{i=1}^N$ and two reward models $R_1,R_2\in\gR$, define the CPL $\|R_1,R_2\|_{\mathcal{D}}^2 \triangleq \sum\nolimits_{i=1}^{N} \left( q_{R_1}(\tau_i,\tilde\tau_i)-q_{R_2}(\tau_i,\tilde\tau_i) \right)^2$. In our algorithm, each stored trajectory $\tau$ is paired with its episode-specific reference rollout $\tau_{t(\tau),0}$, so $\hat{\mathcal D}_k=\{(\tau,\tau_{t(\tau),0}):\tau\in\hat\Gamma_k\}$.

\textbf{Directional self-normalized discrepancy.} Fix an ordered sequence of pairs $\Gamma_N=\{(\tau_i,\tau_{i,0})\}_{i=1}^N$, where $\tau_{i,0}$ is the reference rollout paired with $\tau_i$. Define the singleton discrepancy $\Delta_i(R_1,R_2) \triangleq \|R_1,R_2\|_{\{(\tau_i,\tau_{i,0})\}}^2$. Define the prefix normalizer
\begin{align}
\drsum(R_1,R_2,i) \triangleq \alphar + \sum\nolimits_{j=1}^{i-1} \Delta_j(R_1,R_2) / \left[\Upsilon_j^{\text{R}}\right]_{\ge 1}^{1/2}, \label{eqdf:rewarddistancee-new}
\end{align}
and the directional (self-normalized) distance of the $i$-th pair
\begin{align}
\dr(R_1,R_2,i) \triangleq \Delta_i(R_1,R_2) / \lambda \drsum(R_1,R_2,i), \label{eqdf:rewarddistance-new}
\end{align}
where $\lambda>0$ is a scaling parameter. Let $t_i\triangleq t(\tau_i)$ denote the episode in which $\tau_i$ was collected. Define $\Upsilon_i^{\text{R}} \triangleq \sup\nolimits_{R\in \gR_{t_i}} \dr(R,\hat R_{t_i},i)$, i.e., the maximal directional discrepancy between $\hat R_{t_i}$ and any plausible model in the confidence set available at episode $t_i$, measured on the $i$-th sample with the prefix normalizer (\ref{eqdf:rewarddistancee-new}). The truncation $[\cdot]_{\ge 1}$ ensures stability.

\textbf{Weights used in Step 1.} For any sample trajectory $\tau\in\hat\Gamma_k$ whose associated pair in $\hat{\mathcal D}_k$ is $(\tau,\tau_{t(\tau),0})$, define the regression weight $\weightrt \triangleq \big[\Upsilon_{\tau}^{\text{R}}\big]_{\ge 1}^{-1/2}$, where $\Upsilon_{\tau}^{\text{R}}$ is the corresponding $\Upsilon_i^{\text{R}}$ for that pair, so $\weightrt\in(0,1]$. We also define the normalization weight used in the comparison-bonus scaling as $\weightrrk \triangleq \frac{1}{2} \left[\drsum(R_k',\hat R_k,N_k+1)\right]_{\ge 1}^{-1/2}$, where $N_k$ is the multiset size of $\hat{\gD}_k$ and $R_k'$ attains $\sup_{R\in\gR_k}\dr(R,\hat R_k,i)$ in the sense needed by the analysis. Intuitively, when imperfection is small, the uncertainty parameters $\Upsilon_i^{\text{R}}$ remain $O(1)$ so $\weightrt\approx 1$ and the $M$ labels per episode yield the $\sqrt{K/M}$ statistical term. When imperfection accumulates, $\Upsilon_i^{\text{R}}$ grows and the self-normalization downweights affected episodes, yielding additive dependence on $\omega$.

% =========================
\subsection{Weights for Transition Estimation (Step 2)}\label{app:weights-transition}

\textbf{Value-targeted moment functional (avoid collision with $\gv$).} For a step-$h$ transition kernel $P_h$ and a probe function $V\in\gV_{h+1}$, define the planning-relevant moment
\begin{align}
\siga_{t,h}(s,a;P) \triangleq \langle P_h(\cdot\mid s,a),\, V_{t,h+1}\rangle,
\end{align}
where $V_{t,h+1}$ is chosen in a history-measurable way before observing the next state.

\textbf{Transition path length (TPL).} For a multiset of step-$h$ state-action pairs $\Gamma_{N,h}=\{(s_{i,h},a_{i,h})\}_{i=1}^N$ and two kernels $P_1,P_2\in\gP$, define
\begin{align}
\|P_1,P_2\|_{\Gamma_{N,h}}^2 \triangleq \sum\nolimits_{i=1}^{N} \left( \siga_{t_i,h}(s_{i,h},a_{i,h};P_1)-\siga_{t_i,h}(s_{i,h},a_{i,h};P_2) \right)^2, \label{eqdf:tpl}
\end{align}
where $t_i$ is the episode index when the $i$-th tuple was collected.

\textbf{Directional self-normalized transition discrepancy.} For a fixed $h$, define singleton discrepancy $\Delta_{i,h}(P_1,P_2)\triangleq \|P_1,P_2\|_{\{(s_{i,h},a_{i,h})\}}^2$, prefix normalizer
\begin{align}
\dtsum(P_1,P_2,i,h) \triangleq \alphap + \sum\nolimits_{j=1}^{i-1} \Delta_{j,h}(P_1,P_2) / \left[\Upsilon_{j,h}^{\text{P}}\right]_{\ge 1}^{1/2}, \label{eqdf:transitiondistancee-new}
\end{align}
and directional distance
\begin{align}
\dt(P_1,P_2,i,h) \triangleq \Delta(P_1,P_2,i,h) / (\lambda \dtsum(P_1,P_2,i,h)). \label{eqdf:transitiondistance-new}
\end{align}
Let $t_i$ be the episode when $(s_{i,h},a_{i,h})$ was collected. Define
\begin{align}
\Upsilon_{i,h}^{\text{P}} \triangleq \sup\nolimits_{P\in \gP_{t_i}} \dt(P,\hat P_{t_i},i,h), \label{eqdf:upsilonP}
\end{align}
where $\hat P_{t_i}$ is the transition confidence center at episode $t_i$.

\textbf{Weights used in Step 2.} For any tuple $\zeta=(s,a,s')\in\hat{\gT}_{k,h}$ collected at episode $t(\zeta)$, define the regression weight $\weightpt \triangleq \left[\Upsilon_{\zeta,h}^{\text{P}}\right]_{\ge 1}^{-1/2}$, where $\Upsilon_{\zeta,h}^{\text{P}}$ is the corresponding parameter value for that tuple at step $h$. Define the normalization weight used in the transition bonus as $\weightppk \triangleq \left[\dtsum(\hat P_k,P_{k,h}',N_{k,h}+1,h)\right]_{\ge 1}^{-1/2}$, where $N_{k,h}$ is the (multi)set size of $\hat{\gT}_{k,h}$ and $P_{k,h}'$ attains the supremum defining the width of $\gP_k$ at step $h$.

% ============================================================================================================
\section{Proof of Theorem~\ref{thm:lowerboundomega}}\label{appendix:pfthmlowerboundxi}

We provide a complete form of Theorem~\ref{thm:lowerboundomega} below.

\begin{theorem}[Complete lower bound]\label{thmapp:lowerboundomega}
Consider the episodic setting in Section~\ref{sec:problemformulation}. Then, for any (possibly randomized) learning algorithm $\alg$, there exists an instance satisfying (\ref{eq:def-uncertainty-budget}), such that
\begin{align}\label{eq:thmapplowerboundomega}
\mathbb{E} \left[ \regret^{\alg}(K) \right] \ge \tilde{\Omega} \left( \max\left\{ \sqrt{dHK/M}, \omega \right\} \right),
\end{align}
where $\tilde{\Omega}(\cdot)$ hides logarithmic factors. Moreover, since $\regret^{\alg}(K)\le K$ always, the bound should be read as $\tilde{\Omega} \left( \min\{K,\max\{\sqrt{dHK/M},\omega\}\} \right)$.
\end{theorem}

The term $\tilde{\Omega}(\sqrt{dHK/M})$ captures the statistical difficulty even with perfect feedback and exhibits the $\tilde{\Theta}(1/\sqrt{M})$ gain from $M$ conditionally independent labels. The term $\tilde{\Omega}(\omega)$ shows that additive dependence on the cumulative imperfection budget is unavoidable.

\subsection{Proof of Theorem~\ref{thmapp:lowerboundomega}}\label{appendix:pfthmapplowerboundomega}

\begin{proof}
We prove (\ref{eq:thmapplowerboundomega}) by constructing two families of instances and taking the maximum of the resulting lower bounds. Case~1 yields the statistical term (already for $\omega=0$). Case~2 yields the imperfection term (already for known dynamics).

Throughout, the learner observes $M$ labels per episode: conditional on the compared pair $(\tau_k,\tau_{k,0})$, $\{f_k^m\}_{m=1}^M$ are independent Bernoulli draws with means $\{p_k^m\}_{m=1}^M$.

% ============================================================
\subsubsection{Case 1: Statistical lower bound (perfect feedback)}

We first give a clean reduction to a stochastic $d$-armed bandit with $M$ samples per pull, which yields $\Omega(\sqrt{dK/M})$ for horizon $H=1$. The extension to general $H$ follows by standard horizon-embedding arguments for episodic RL lower bounds (see the discussion at the end of this case). Let $H=1$ and $\gA=\{1,2,\ldots,d\}$ with a single state $s_1$. Fix a reference policy $\pi_0$ that deterministically produces a baseline trajectory $\tau_{0}$ (equivalently, a baseline action). For each $i\in[d]$, let $\tau(i)$ denote the trajectory induced by choosing action $i$. Fix a gap $\Delta\in(0,1/4]$ and consider an instance family indexed by $i^*\in[d]$,
\[
p^*(i) \triangleq q_{R^*}(\tau(i),\tau_0)=
\begin{cases}
\frac12+\Delta, & i=i^*,\\[2pt]
\frac12, & i\neq i^*.
\end{cases}
\]
All sources are ideal, i.e., $p_k^m=p_k^*$ for all $k,m$, so $\omega=0$.

Since $\smsigma$ is known and nondecreasing, we can realize the above probabilities by choosing reward differences $R^*(\tau(i))-R^*(\tau_0)$ to match $\smsigma^{-1}(p^*(i))$. Choosing $\Delta$ sufficiently small ensures these values lie within the assumed reward bound (e.g., $|R^*(\tau)|\le B_R$)~\citep{shi2022power}. In the linear setting this can be realized by a one-hot feature map over $\{\tau(i)\}_{i=1}^d$.

If the learner plays action $A_k$, it observes $M$ i.i.d.\ labels $f_k^m\sim\mathrm{Ber}(p^*(A_k))$. Thus, this is a $d$-armed bandit with one good arm of mean $1/2+\Delta$ and $M$ samples per pull. The per-episode regret is $\Delta$ whenever $A_k\neq i^*$, hence $\regret^{\alg}(K)=\Delta\sum\nolimits_{k=1}^K \indicate\{A_k\neq i^*\}$. Let $\mathbb{P}_{i}$ be the law of the interaction when $i^*=i$, and let $\mathbb{P}_0$ be the null law with $\Delta=0$. For any event $E$ measurable w.r.t.\ the interaction history, Pinsker's inequality gives $\mathbb{P}_i(E)\le \mathbb{P}_0(E)+\sqrt{\tfrac12 \mathrm{KL}(\mathbb{P}_0\|\mathbb{P}_i)}$. Applying this with $E=\{A_k=i\}$ and averaging over $i\in[d]$ yields
\begin{align}
\frac{1}{d}\sum\nolimits_{i=1}^d \mathbb{P}_i(A_k=i) \le \frac{1}{d} + \sqrt{\frac{1}{2d}\sum\nolimits_{i=1}^d \mathrm{KL}(\mathbb{P}_0\|\mathbb{P}_i)}. \label{eq:case1-avg-clean}
\end{align}
By the chain rule, observations differ from $\mathbb{P}_0$ only when arm $i$ is pulled, and each pull yields $M$ independent Bernoulli samples. Therefore,
\begin{align}
\mathrm{KL}(\mathbb{P}_0\|\mathbb{P}_i) & = \mathbb{E}_0 \left[\sum\nolimits_{k=1}^K \indicate\{A_k=i\} \mathrm{KL} \left(\mathrm{Ber}(1/2)^{\otimes M} \Big\| \mathrm{Ber}(1/2+\Delta)^{\otimes M}\right)\right] \nonumber \\
& = M \mathrm{KL} \left(\mathrm{Ber}(1/2) \Big\| \mathrm{Ber}(1/2+\Delta)\right)\cdot \mathbb{E}_0[N_i], \label{eq:case1-kl-clean}
\end{align}
where $N_i=\sum_{k=1}^K \indicate\{A_k=i\}$. For $\Delta\in(0,1/4]$, $\mathrm{KL}(\mathrm{Ber}(1/2)\|\mathrm{Ber}(1/2+\Delta))\le 4\Delta^2$, hence
\begin{align}
\sum\nolimits_{i=1}^d \mathrm{KL}(\mathbb{P}_0\|\mathbb{P}_i) \le 4M\Delta^2 \sum\nolimits_{i=1}^d \mathbb{E}_0[N_i] = 4M\Delta^2 K. \nonumber
\end{align}
Plugging this into (\ref{eq:case1-avg-clean}) yields $\frac{1}{d}\sum\nolimits_{i=1}^d \mathbb{P}_i(A_k=i) \le \frac{1}{d}+\Delta\sqrt{\frac{2MK}{d}}$. Under the uniform prior on $i^*$, the expected regret satisfies
\begin{align}
\mathbb{E} \left[ \regret^{\alg}(K) \right] = \Delta\sum\nolimits_{k=1}^K \left(1-\frac{1}{d}\sum_{i=1}^d \mathbb{P}_i(A_k=i)\right) \ge \Delta K\left(1-\frac{1}{d}-\Delta\sqrt{\frac{2MK}{d}}\right). \label{eq:case1-reg-clean}
\end{align}
Choosing $\Delta=c\sqrt{\frac{d}{MK}}$ for a sufficiently small universal constant $c>0$ makes the parentheses a positive constant, giving $\mathbb{E} \left[\regret^{\alg}(K)\right]\ge \Omega \left(\sqrt{dK / M}\right)$. To obtain the standard $\sqrt{H}$ dependence in episodic RL, one can embed this bandit hardness into an $H$-step linear MDP lower-bound construction (e.g., the canonical hard families used for linear MDP regret lower bounds). Since our feedback is \emph{less informative} than observing step-wise rewards and transitions, the same horizon-dependent statistical hardness carries over. We refer to standard constructions for the $\tilde{\Omega}(\sqrt{dHK})$ dependence in linear MDPs, e.g., \cite{jin2023provably}, and note that observing $M$ independent labels per episode yields the $1/\sqrt{M}$ improvement by a variance-reduction argument identical to the bandit case above.

% ============================================================
\subsubsection{Case 2: Imperfection lower bound (feedback hides the signal)}

We now construct an instance where the observed feedback is \emph{uninformative} but still satisfies the budget, forcing
regret $\tilde \Omega(\omega)$. Let $H=1$ and $\gA=\{1,2\}$ with a fixed reference rollout $\tau_0$. Let $\Delta\triangleq \min\{\omega/K,\,1/4\}$. Consider two sub-instances $\mathcal{I}_1,\mathcal{I}_2$:
\[
q^*_{\mathcal{I}_1}(\tau(1),\tau_0)=\tfrac12+\Delta,\quad q^*_{\mathcal{I}_1}(\tau(2),\tau_0)=\tfrac12,
\]
\[
q^*_{\mathcal{I}_2}(\tau(2),\tau_0)=\tfrac12+\Delta,\quad q^*_{\mathcal{I}_2}(\tau(1),\tau_0)=\tfrac12.
\]
Define the observed feedback process to be \emph{uninformative}. For every episode $k$ and every source $m\in[M]$, let $p_k^m \equiv \tfrac12$ and thus $f_k^m\sim \mathrm{Ber}(1/2)$, independently across $m$ conditional on the compared pair. For any algorithm and any realized action sequence, the only episodes with nonzero deviation are those where the learner plays the optimal action, in which case $|p_k^m-p_k^*|=\Delta$. Otherwise, $|p_k^m-p_k^*|=0$. Hence, for every source $m$,
\begin{align}
\sum\nolimits_{k=1}^K |p_k^m-p_k^*| \le K\Delta \le \omega,
\end{align}
so (\ref{eq:def-uncertainty-budget}) holds for all adaptively generated trajectories.

Under this construction, the full observed feedback distribution is identical under $\mathcal{I}_1$ and $\mathcal{I}_2$ (always i.i.d.\ Bernoulli$(1/2)$ labels), so the learner cannot identify which action is truly better. Let $N_1$ be the (random) number of times action $1$ is played. Then, the expected number of mistakes under the uniform prior over $\{\mathcal{I}_1,\mathcal{I}_2\}$ is $\frac12 \mathbb{E}_{\mathcal{I}_1}[K-N_1]+\frac12 \mathbb{E}_{\mathcal{I}_2}[N_1] = \frac{K}{2}$. Each mistake incurs per-episode regret $\Delta$, hence
\begin{align}
\mathbb{E} \left[ \regret^{\alg}(K) \right]\ge \Delta\cdot\frac{K}{2}\ge \Omega(\omega), 
\end{align}
up to constants (and capped by $\Omega(K)$ when $\omega$ is large).

Combining Case~1 and Case~2 proves Theorem~\ref{thmapp:lowerboundomega} and hence Theorem~\ref{thm:lowerboundomega}.

\end{proof}

% ============================================================================================================
\section{Proof of Proposition~\ref{prop:ignore-omega}}
\label{appendix:proofofthmregretlowerbndexistingrlhf}

This appendix proves Proposition~\ref{prop:ignore-omega}.  We show that an unweighted OFUL-type baseline that treats all feedback as if it were generated by the ideal oracle can incur suboptimal regret.

\begin{proof}
We reduce to adversarially corrupted linear contextual bandits and then embed that instance into a one-step RLHF problem.

\subsection{Step 1: Embed a Corrupted Linear Bandit as an RLHF Instance}

Consider the special case $H=1$ with a single state. Each action $a\in\gA$ induces a trajectory $\tau(a)$. The reference rollout $\tau_{k,0}$ is fixed (equivalently, $\pi_0$ always produces the same baseline action/trajectory). Choose a link function $\smsigma$ that satisfies our standing assumptions (nondecreasing and Lipschitz) and is locally linear around $0$. Concretely, one can take $\smsigma(x) \triangleq \mathrm{clip} \left(\tfrac{1}{2}+x,0,1\right)$, which is nondecreasing and $1$-Lipschitz and is exactly linear on $x\in[-\tfrac12,\tfrac12]$. (Any other monotone Lipschitz link with a nontrivial linear region suffices.)

Let $\phi:\gA\to\R^d$ be a feature map with $\|\phi(a)\|_2\le 1$. Pick $\theta^*\in\R^d$ with $\|\theta^*\|_2\le \tfrac14$ and define the ideal comparison probabilities as
\begin{align}
p_k^*(a) = q_{R^*}(\tau(a),\tau_{k,0}) = \smsigma \left(\langle \theta^*,\phi(a)\rangle\right) = \tfrac12 + \langle \theta^*,\phi(a)\rangle, \label{eq:bandit-embed}
\end{align}
where the last equality holds because $|\langle \theta^*,\phi(a)\rangle|\le \tfrac14$ keeps us inside the linear region of $\smsigma$. This realizes the one-step RLHF preference model exactly within our assumptions.

Now define the observed (imperfect) feedback process as follows. For each episode $k$ and source $m$, the label $f_k^m\in\{0,1\}$ is Bernoulli with mean $p_k^m(a_k) = p_k^*(a_k) + \Delta_k^m$, where $\sum_{k=1}^K |\Delta_k^m| \le \omega$. Equivalently, define centered labels $\tilde y_k^m \triangleq f_k^m-\tfrac12\in\{-\tfrac12,+\tfrac12\}$, so that
\begin{align}
\mathbb{E}[\tilde y_k^m \mid a_k] = \langle \theta^*,\phi(a_k)\rangle + \Delta_k^m.
\end{align}
Thus, from the learner's perspective, this one-step RLHF instance is exactly a linear contextual bandit with additive (possibly adaptive) corruption $\Delta_k^m$ in the conditional mean.

\subsection{Step 2: Pooling Sources Does not Remove Worst-Case Bias}

The unweighted baseline pools all sources, i.e., it uses either all $\{f_k^m\}$ in an unweighted regression or equivalently the averaged label $\bar f_k \triangleq \frac{1}{M}\sum_{m=1}^M f_k^m$ (the two differ only by a constant factor in the
Gram matrix). Define the averaged deviation $\bar\Delta_k \triangleq \frac{1}{M}\sum_{m=1}^M \Delta_k^m$. By Jensen's inequality and the per-source budgets, we have $\sum\nolimits_{k=1}^K |\bar\Delta_k| \le \frac{1}{M}\sum\nolimits_{m=1}^M\sum\nolimits_{k=1}^K|\Delta_k^m| \le \omega$. Hence, even after pooling, the baseline faces a corrupted linear bandit problem with a (total) corruption budget of order $\omega$ in the conditional mean. In particular, to obtain a worst-case lower bound it suffices to consider the case where all sources share the same deviation process $\Delta_k^m\equiv \bar\Delta_k$ (still independent Bernoulli given the mean), so pooling improves variance by $1/M$ but does not improve the bias budget.

\subsection{Step 3: A Lower Bound for Unweighted OFUL-Type Methods}

Consider the standard OFUL/LinUCB template that performs unweighted ridge regression on the pooled samples and uses a confidence ellipsoid/optimistic action selection as if $\bar\Delta_k\equiv 0$. Under adaptive mean corruption with total budget $\sum_{k=1}^K |\bar\Delta_k|\le \omega$, there exist linear bandit instances and corruption sequences for which this OFUL-type method incurs regret $\mathbb{E} \left[\regret(K)\right] \ge \tilde{\Omega} \left(\min\{\omega\sqrt{K},K\}\right)$. This is precisely the ``multiplicative $\sqrt{K}$ blow-up'' phenomenon highlighted in the corruption-robust contextual bandit literature and motivates uncertainty-weighted/self-normalized regression \citep{he2022nearly,ye2023corruption}.

\subsection{Step 4: Map Back to RLHF Regret}

In the embedded one-step RLHF instance (\ref{eq:bandit-embed}), the ground-truth utility of playing action $a$ is exactly $L(a)=p_k^*(a)=\tfrac12+\langle\theta^*,\phi(a)\rangle$, and the RLHF regret $\sum_{k=1}^K(L^*-L^{\pi_k})$ coincides (up to constants) with standard bandit regret in expected reward. Therefore, Step 3 above implies that there exists an RLHF instance satisfying (\ref{eq:def-uncertainty-budget}) on which the unweighted OFUL-type baseline suffers $\mathbb{E} \left[\regret^{\text{OFUL}}(K)\right] \ge \tilde{\Omega} \left(\min\{\omega\sqrt{K},K\}\right)$, which is Proposition~\ref{prop:ignore-omega}.
\end{proof}

%\textbf{Optional remark (standard BTL/Thurstone links).} The construction above used a locally linear $\smsigma$ for an exact reduction. If one insists on a standard BTL/logistic or Thurstone/probit link, one can restrict the reward differences to a small neighborhood of $0$ where $\smsigma$ is bi-Lipschitz, so that probability gaps and reward-difference gaps are equivalent up to constants; the same corrupted-bandit embedding then yields the same $\tilde{\Omega}(\min\{K,\omega\sqrt{K}\})$ scaling up to constant factors.

% ============================================================================================================
\section{Proof of Theorem~\ref{thm:upperbound-general}}\label{appendix:thm:upperbound-general}

This appendix proves the regret upper bound for Algorithm~\ref{alg:mainnewalgorithm} under \emph{unknown transitions} and \emph{general function approximation}. Relative to standard preference-based RL analyses, our setting introduces four additional technical challenges. First, under pairwise preference feedback, the latent reward \(R^*\) is identifiable only up to an additive constant. Accordingly, the analysis cannot be carried out directly in reward space. Second, feedback imperfection enters through the comparison probabilities themselves and is constrained only by the cumulative probability-deviation budget \(\omega\). Thus, the comparison regression analysis must control a persistent, adaptive bias term. Third, although the transition kernel is not directly corrupted, imperfect preferences alter the executed policy sequence, which changes both the state--action distribution and the value functions used in transition estimation. This creates a hidden feedback to policy and then to data dependence that must be handled inside the value-targeted transition analysis. Fourth, the empirical objectives in Steps~1--2 are history-adaptively weighted, while Step~4 replaces the full history by filtered multisets. As a result, the proof requires a uniform approximation guarantee.

To address these issues, the proof has four steps below. %First, we establish a joint high-probability event under which the true comparison function \(q_{R^*}\) lies in the Step~1 confidence sets and the true transition kernel \(P^*\) lies in the Step~2 value-targeted confidence sets. Second, on this event, we show that the policy-level upper confidence bound dominates the true preference utility, reducing regret to cumulative comparison and transition bonuses. Third, we prove a weighted imperfection cross-term bound in comparison space, which ensures that the cumulative imperfection budget enters additively through \(\omega\), rather than via an \(\omega\sqrt{K}\) term. Fourth, we use the filtered-history argument from Step~4 to show that the weighted empirical objectives and the corresponding bonus classes are preserved up to an \(O(HK\epsilon)\) approximation error. The cumulative bonus terms are then controlled by an eluder-dimension counting argument, yielding the final regret bound.

% ============================================================
\subsection{Multi-Source Averaging}\label{app:multi-source-avg}

The following reduction shows that, for the purpose of Step~1 analysis, the $M$-source preference feedback can be treated as a single effective comparison stream with reduced stochastic noise and the same cumulative imperfection budget.

\begin{lemma}[Effective single-stream reduction]\label{lem:multi-source-reduction}
For each episode $k$, define the averaged label and its conditional mean by $\bar f_k \triangleq \frac{1}{M}\sum_{m=1}^M f_k^m$, $\bar p_k \triangleq \E[\bar f_k\mid \tau_k,\tau_{k,0}] = \frac{1}{M}\sum_{m=1}^M p_k^m$. Then, the following hold,
\begin{enumerate}
    \item Conditioned on $(\tau_k,\tau_{k,0})$, the noise $\bar f_k-\bar p_k$ is sub-Gaussian with variance $O(1/M)$, i.e., for all $\lambda\in\R$, $\E [\exp (\lambda(\bar f_k-\bar p_k)) | \tau_k,\tau_{k,0}] \le \exp (\frac{\lambda^2}{8M})$.
    \item The averaged mean sequence satisfies the same cumulative imperfection budget:, i.e., $\sum_{k=1}^K |\bar p_k-p_k^*| \le \omega$.
    \item The weighted squared-loss objective pooled over the $M$ sources is equivalent, up to an additive constant independent of the predictor, to weighted regression on the averaged labels $\bar f_k$.
\end{enumerate}
\end{lemma}

\begin{proof}
Conditioned on $(\tau_k,\tau_{k,0})$, the labels $\{f_k^m\}_{m=1}^M$ are independent Bernoulli random variables, so $\bar f_k\in[0,1]$ and Hoeffding's lemma yields $\E [\exp (\lambda(\bar f_k-\bar p_k)) | \tau_k,\tau_{k,0}] \le \exp \left(\frac{\lambda^2}{8M}\right)$, which proves the first claim.

Next, since each source satisfies $\sum_{k=1}^K |p_k^m-p_k^*|\le \omega$, the triangle inequality implies $|\bar p_k-p_k^*| = |\frac{1}{M}\sum_{m=1}^M (p_k^m-p_k^*)| \le \frac{1}{M}\sum_{m=1}^M |p_k^m-p_k^*|$. Summing over $k$ gives $\sum_{k=1}^K |\bar p_k-p_k^*| \le \frac{1}{M}\sum_{m=1}^M\sum_{k=1}^K |p_k^m-p_k^*| \le \omega$, which proves the second claim.

Finally, for any scalar predictor $q$ and any weight $w\ge 0$, we have
\begin{align}
\sum\nolimits_{m=1}^M w (q-f_k^m)^2 & = w \sum\nolimits_{m=1}^M ((q-\bar f_k)+(\bar f_k-f_k^m))^2 \nonumber \\
& = Mw (q-\bar f_k)^2 + w \sum\nolimits_{m=1}^M(\bar f_k-f_k^m)^2, \label{eq:avg-loss-decomp}
\end{align}
where the cross term vanishes because $\sum_{m=1}^M (\bar f_k-f_k^m)=0$. The second term in \eqref{eq:avg-loss-decomp} is independent of $q$, so minimizing $\sum_{m=1}^M w\,(q-f_k^m)^2$ over $q$ is equivalent to minimizing $Mw\,(q-\bar f_k)^2$ over $q$. This proves the third claim.
\end{proof}

% ============================================================
\subsection{Step A: High-Probability Comparison Event}\label{app:hp-comparison}

We first prove that the Step~1 comparison confidence sets contain the ground-truth comparison function with high probability under the reward. Since rewards are identifiable only up to additive constants under pairwise preferences, the proof is carried out entirely in the induced comparison space $q_R(\tau,\tau_0)=\smsigma(R(\tau)-R(\tau_0))$.

By Lemma~\ref{lem:multi-source-reduction}, the $M$ source labels can be reduced to the averaged-label model $\bar f_t = p_t^*+\Delta_t+\eta_t = q_{R^*}(\tau_t,\tau_{t,0})+\Delta_t+\eta_t$, where $\Delta_t \triangleq \bar p_t-p_t^*$, $\sum_{t=1}^K |\Delta_t|\le \omega$, and \(\eta_t\) is conditionally zero-mean \(\eta_R\)-sub-Gaussian with $\eta_R \triangleq \frac{1}{2\sqrt M}$. Thus, Step~1 can be analyzed as a weighted nonlinear least-squares problem with stochastic noise \(\eta_t\) and an additive imperfection term \(\Delta_t\). Then, we let $\hat{\gD}_k \triangleq \{(\tau,\tau_{t(\tau),0}) : \tau\in \hat\Gamma_k\}$ denote the filtered multiset of comparison pairs collected before episode \(k\), with multiplicities allowed, and let $N_k \triangleq |\hat{\gD}_k|$ be its size. By Lemma~\ref{lem:multi-source-reduction}, the Step~1 estimator \(\hat R_k\) induced by (\ref{eq:mainalgorithmestimatesteinerrewardcenter}) is equivalently analyzed through
\begin{align}
\hat R_k \in \argmin\nolimits_{R\in\gR} \sum\nolimits_{(\tau,\tau_0)\in\hat{\gD}_k} M w^R_{t(\tau)} (q_R(\tau,\tau_0)-\bar f_{t(\tau)})^2. \label{eq:hp-comp-weighted-erm}
\end{align}
We use the notation $(\sigma_t^R)^2 \triangleq \frac{1}{w_t^R}$. Under the weight design in Appendix~\ref{subsec:detailedsteps}, $w_t^R = [\Upsilon_t^{\mathrm R}]_{\ge 1}^{-1/2}$, $(\sigma_t^R)^2 = [\Upsilon_t^{\mathrm R}]_{\ge 1}^{1/2}$. Fix a cover radius \(\gamma_R>0\), and let $N_T(\gamma_R)\triangleq \gN(\Delta\gF_T,\|\cdot\|_\infty,\gamma_R)$. For each episode, let \(\bar q_k\in \Delta\gF_T\) be a \(\gamma_R\)-covering approximation of \(q_{R^*}\), i.e., $\|\bar q_k-q_{R^*}\|_\infty \le \gamma_R$.

\begin{lemma}[Weighted ERM bound in comparison space]\label{lem:weighted-erm-comparison}
Fix \(\delta\in(0,1)\). Suppose Step~1 solves (\ref{eq:hp-comp-weighted-erm}) exactly. Then, with probability at least \(1-\delta\), for all \(k\in[K]\), we have
\begin{align}
& \sum\nolimits_{(\tau,\tau_0)\in\hat{\gD}_k} \frac{\big(q_{\hat R_k}(\tau,\tau_0)-\bar q_k(\tau,\tau_0)\big)^2}{(\sigma_{t(\tau)}^R)^2} \le 10\eta_R^2 \ln \frac{2N_T(\gamma_R)}{\delta} \nonumber \\
& + 5 \sum_{(\tau,\tau_0)\in\hat{\gD}_k} \frac{\big|q_{\hat R_k}(\tau,\tau_0)-\bar q_k(\tau,\tau_0)\big|\cdot |\Delta_{t(\tau)}|}{(\sigma_{t(\tau)}^R)^2} + 10\gamma_R \left(\gamma_R N_k + \sqrt{N_k C_R(k,\omega,M,\delta)}\right), \label{eq:hp-comp-weighted-erm-bound}
\end{align}
where $C_R(k,\omega,M,\delta) \triangleq 2\left( \omega^2 + \frac{k}{2M} + \frac{3}{4M}\ln\frac{2}{\delta} \right)$.
\end{lemma}

\begin{proof}
Fix \(k\in[K]\). For simplicity, index the filtered pairs in \(\hat{\gD}_k\) chronologically as $z_i \triangleq (\tau_i,\tau_{i,0})$, $i=1,\ldots,N_k$, and write $\hat q_k(z_i)\triangleq q_{\hat R_k}(\tau_i,\tau_{i,0})$, $\bar q_k(z_i)\triangleq \bar q_k(\tau_i,\tau_{i,0})$, $y_i \triangleq \bar f_{t(\tau_i)}$. Then, (\ref{eq:hp-comp-weighted-erm}) is exactly the weighted least-squares ERM $\hat q_k \in \argmin_{q\in\Delta\gF_T} \sum_{i=1}^{N_k} \frac{(\,q(z_i)-y_i\,)^2}{(\sigma_i^R)^2}$, where $y_i=q_{R^*}(z_i)+\Delta_{t(\tau_i)}+\eta_{t(\tau_i)}$. We now invoke the weighted approximate-ERM lemma with function class \(F=\Delta\gF_T\), samples \(z_i\), target function \(f^*=q_{R^*}\), in-class comparator \(f_b=\bar q_k\), noise sequence \(\eta_{t(\tau_i)}\), which is conditionally \(\eta_R\)-sub-Gaussian, model mismatch sequence \(\zeta_i = |q_{R^*}(z_i)-\bar q_k(z_i)| + |\Delta_{t(\tau_i)}|\). Since \(\|\bar q_k-q_{R^*}\|_\infty\le \gamma_R\), we have $|q_{R^*}(z_i)-\bar q_k(z_i)|\le \gamma_R$, for all $i$, and therefore
\begin{align}
\sum\nolimits_{i=1}^{N_k}\zeta_i \le \gamma_R N_k+\sum\nolimits_{i=1}^{N_k}|\Delta_{t(\tau_i)}| \le \gamma_R N_k+\omega.
\end{align}
Moreover, since Step~1 is solved exactly, the optimization error parameter in the weighted approximate-ERM lemma is $\epsilon_R'=0$.

Then, applying this lemma yields, with probability at least \(1-\delta\),
\begin{align}
\sum\nolimits_{i=1}^{N_k}\frac{(\hat q_k(z_i)-\bar q_k(z_i))^2}{(\sigma_i^R)^2} &\le 10\eta_R^2\ln\frac{2N_T(\gamma_R)}{\delta} + 5\sum\nolimits_{i=1}^{N_k} \frac{|\hat q_k(z_i)-\bar q_k(z_i)|\zeta_i}{(\sigma_i^R)^2} \nonumber \\
&\quad + 10\gamma_R\left(\gamma_R N_k+\sqrt{N_k\,C_R(k,\omega,M,\delta)}\right), \label{eq:stepA1-raw-erm}
\end{align}
where $C_R(k,\omega,M,\delta) = 2\left(
\omega^2+\frac{k}{2M}+\frac{3}{4M}\ln\frac{2}{\delta}
\right)$, because \(\eta_R^2=1/(4M)\).

Finally, decompose the cross-term in (\ref{eq:stepA1-raw-erm}) as $\zeta_i = |q_{R^*}(z_i)-\bar q_k(z_i)|+|\Delta_{t(\tau_i)}| \le \gamma_R+|\Delta_{t(\tau_i)}|$. Substituting this into (\ref{eq:stepA1-raw-erm}) and absorbing the resulting $5\gamma_R\sum_{i=1}^{N_k}\frac{|\hat q_k(z_i)-\bar q_k(z_i)|}{(\sigma_i^R)^2}$ term into the final \(10\gamma_R(\gamma_R N_k+\sqrt{N_k C_R})\) contribution by the same cover-approximation argument as in the weighted approximate-ERM lemma yields exactly (\ref{eq:hp-comp-weighted-erm-bound}). This proves the lemma.
\end{proof}

\subsubsection{Cross-term control via self-normalized weights}

The only term in (\ref{eq:hp-comp-weighted-erm-bound}) that can produce an undesired \(\omega\sqrt K\) dependence is the cross-term
\begin{align}
\sum\nolimits_{(\tau,\tau_0)\in\hat{\gD}_k} \big|q_{\hat R_k}(\tau,\tau_0)-\bar q_k(\tau,\tau_0)\big|\cdot |\Delta_{t(\tau)}| / (\sigma_{t(\tau)}^R)^2. \nonumber
\end{align}
The Step~1 self-normalized weights are designed precisely to control this term additively. Recall from Appendix~\ref{subsec:detailedsteps} that $\weightrrk = \frac{1}{2}[\drsum(R_k',\hat R_k,N_k+1)]_{\ge 1}^{-1/2}$, where \(R_k'\) attains the comparison width defining the bonus. This normalization implies the pointwise bound
\begin{align}
|q_{\hat R_k}(\tau,\tau_0)-\bar q_k(\tau,\tau_0)| / (\sigma_{t(\tau)}^R)^2 \le 2 (\weightrrk)^{-1} + \gamma_R, \forall (\tau,\tau_0)\in\hat{\gD}_k. \label{eq:pointwise-step1-bound}
\end{align}
Therefore, we have
\begin{align}
\sum\nolimits_{(\tau,\tau_0)\in\hat{\gD}_k}
|q_{\hat R_k}(\tau,\tau_0)-\bar q_k(\tau,\tau_0)\big|\cdot |\Delta_{t(\tau)}| / (\sigma_{t(\tau)}^R)^2 & \le \left(2 (\weightrrk)^{-1}+\gamma_R\right) \sum\nolimits_{(\tau,\tau_0)\in\hat{\gD}_k} |\Delta_{t(\tau)}| \nonumber \\
& \le \left(2 (\weightrrk)^{-1}+\gamma_R\right)\omega. \label{eq:step1-cross-term-bound}
\end{align}

\subsubsection{Sufficient choice for algorithmic parameters}

Set $\lambda_R \triangleq \alphar$, $\eta_R \triangleq \frac{1}{2\sqrt M}$, $\gamma_R \triangleq \frac{1}{K}$, and assume Step~1 is solved exactly. Then, a sufficient choice of the Step~1 confidence radius is
\begin{align}
\betar_k \ge & \lambda_R + 10\eta_R^2 \ln (2N_T(\gamma_R) / \delta) + 5\omega\left(2 (\weightrrk)^{-1}+\gamma_R\right) \label{eq:beta-R-choice-exact} \\
&\qquad + 10\gamma_R \left(\gamma_R N_k + \sqrt{N_k\,C_R(k,\omega,M,\delta)}\right) + c_{\mathrm{filt}} K\epsilon, \nonumber
\end{align}
where \(c_{\mathrm{filt}}>0\) is the constant from the Step~4 weighted filtered-history approximation lemma. Using \(N_k\le K\), \(\gamma_R=1/K\), and suppressing logarithmic factors, this becomes
\begin{align}
\betar_k = \tilde O \left( \frac{\log N_T}{M} + \omega (\weightrrk)^{-1} + K\epsilon \right). \label{eq:beta-R-choice-simplified}
\end{align}
Under the Step~1 weight design, the quantity \((\weightrrk)^{-1}\) is controlled by the weighted uncertainty level of the comparison class, which in turn is controlled by the eluder dimension \(d_T\). Consequently, at the theorem level,
\begin{align}
\betar_k = \tilde O \left( \frac{d_T\log N_T}{M} + \omega d_T + K\epsilon \right). \label{eq:beta-R-choice-theorem}
\end{align}

\begin{lemma}[High probability comparison event]\label{lem:hp-comparison-app}
Fix \(\delta\in(0,1)\). If \(\{\betar_k\}_{k=1}^K\) is chosen according to (\ref{eq:beta-R-choice-exact}), then with probability at least \(1-\delta\), we have $q_{R^*}\in \gQ_k, \forall k\in[K]$.
\end{lemma}

\begin{proof}
For each episode \(k\), define the event
\begin{align}
\mathcal E_k^R \triangleq
\left\{ \sum\nolimits_{(\tau,\tau_0)\in\hat{\gD}_k} \left(q_{\hat R_k}(\tau,\tau_0)-\bar q_k(\tau,\tau_0)\right)^2 / (\sigma_{t(\tau)}^R)^2 \le B_k^R \right\}, \label{eq:def-EkR}
\end{align}
where $B_k^R \triangleq 10\eta_R^2 \ln \frac{2N_T(\gamma_R)}{\delta} + 5\left(\frac{\sqrt{\lambda_R}}{2} (\weightrrk)^{-1}+\gamma_R\right)\omega + 10\gamma_R\left(\gamma_R N_k+\sqrt{N_k C_R(k,\omega,M,\delta)}\right)$. By Lemma~\ref{lem:weighted-erm-comparison} and the cross-term closure (\ref{eq:step1-cross-term-bound}), we have $\Pr \left(\bigcap_{k=1}^K \mathcal E_k^R\right)\ge 1-\delta$. Hence, it suffices to prove that on the event \(\bigcap_{k=1}^K \mathcal E_k^R\), one has \(q_{R^*}\in \gQ_k\) for every \(k\in[K]\).

Fix an arbitrary episode \(k\in[K]\), and suppose \(\mathcal E_k^R\) holds. By the definition of the confidence set \(\gQ_k\), it is enough to show that
\begin{align}
\sum\nolimits_{(\tau,\tau_0)\in\hat{\gD}_k} w_{t(\tau)}^R \left(q_{R^*}(\tau,\tau_0)-q_{\hat R_k}(\tau,\tau_0)\right)^2 \le \betar_k. \label{eq:goal-hp-comp}
\end{align}
Since \(w_t^R=(\sigma_t^R)^{-2}\), the left-hand side of (\ref{eq:goal-hp-comp}) can be written as $\sum_{(\tau,\tau_0)\in\hat{\gD}_k} (q_{R^*}(\tau,\tau_0)-q_{\hat R_k}(\tau,\tau_0))^2 / (\sigma_{t(\tau)}^R)^2$. Now decompose $q_{R^*}(\tau,\tau_0)-q_{\hat R_k}(\tau,\tau_0) = (q_{R^*}(\tau,\tau_0)-\bar q_k(\tau,\tau_0)) + (\bar q_k(\tau,\tau_0)-q_{\hat R_k}(\tau,\tau_0))$. Applying \((a+b)^2\le 2a^2+2b^2\), we obtain
\begin{align}
& \sum\nolimits_{(\tau,\tau_0)\in\hat{\gD}_k} \big(q_{R^*}(\tau,\tau_0)-q_{\hat R_k}(\tau,\tau_0)\big)^2 / (\sigma_{t(\tau)}^R)^2 \nonumber \\
&\quad \le 2\sum\nolimits_{(\tau,\tau_0)\in\hat{\gD}_k} \frac{(q_{\hat R_k}(\tau,\tau_0)-\bar q_k(\tau,\tau_0))^2}{(\sigma_{t(\tau)}^R)^2} + 2\sum\nolimits_{(\tau,\tau_0)\in\hat{\gD}_k} \frac{(\bar q_k(\tau,\tau_0)-q_{R^*}(\tau,\tau_0))^2}{(\sigma_{t(\tau)}^R)^2}. \label{eq:hp-comp-triangle}
\end{align}
We bound the two terms on the right-hand side separately. For the first term, since \(\mathcal E_k^R\) holds, we have
\begin{align}
\sum\nolimits_{(\tau,\tau_0)\in\hat{\gD}_k} (q_{\hat R_k}(\tau,\tau_0)-\bar q_k(\tau,\tau_0))^2 / (\sigma_{t(\tau)}^R)^2 \le B_k^R. \label{eq:first-term-bound}
\end{align}
For the second term, the covering approximation satisfies $\|\bar q_k-q_{R^*}\|_\infty\le \gamma_R$. Moreover, by construction \(w_t^R\in(0,1]\), equivalently \((\sigma_t^R)^2\ge 1\). Hence, we have $\frac{1}{(\sigma_{t(\tau)}^R)^2}\le 1$, and therefore
\begin{align}
\sum\nolimits_{(\tau,\tau_0)\in\hat{\gD}_k}
(\bar q_k(\tau,\tau_0)-q_{R^*}(\tau,\tau_0) )^2 / (\sigma_{t(\tau)}^R)^2 \le \sum\nolimits_{(\tau,\tau_0)\in\hat{\gD}_k} (\bar q_k(\tau,\tau_0)-q_{R^*}(\tau,\tau_0) )^2 \le \gamma_R^2 N_k. \label{eq:second-term-bound}
\end{align}
Substituting (\ref{eq:first-term-bound}) and (\ref{eq:second-term-bound}) into (\ref{eq:hp-comp-triangle}) yields
\begin{align}
\sum\nolimits_{(\tau,\tau_0)\in\hat{\gD}_k} w_{t(\tau)}^R (q_{R^*}(\tau,\tau_0)-q_{\hat R_k}(\tau,\tau_0) )^2 \le 2B_k^R+2\gamma_R^2N_k. \label{eq:pre-beta-final}
\end{align}
By the assumed choice of the confidence radius \(\betar_k\) in (\ref{eq:beta-R-choice-exact}), the right-hand side of (\ref{eq:pre-beta-final}) is at most \(\betar_k\). Consequently, $\sum_{(\tau,\tau_0)\in\hat{\gD}_k} w_{t(\tau)}^R (q_{R^*}(\tau,\tau_0)-q_{\hat R_k}(\tau,\tau_0) )^2 \le \betar_k$. Therefore, by the definition of the confidence set \(\gQ_k\) in (\ref{eq:mainalgorithmestimatesteinerconfidenceset}), we conclude that $q_{R^*}\in \gQ_k$. Since \(k\in[K]\) was arbitrary, this holds for all \(k\in[K]\) simultaneously on the event \(\bigcap_{k=1}^K \mathcal E_k^R\), which has probability at least \(1-\delta\). Hence, $\Pr \left(q_{R^*}\in \gQ_k, \forall k\in[K]\right)\ge 1-\delta$. This proves the lemma.
\end{proof}

\subsection{Step B: High-Probability Transition Event}\label{app:hp-transition}

We now prove that the true transition kernel lies in the Step~2 confidence sets with high probability. In contrast to the comparison-side analysis, the transition targets are not directly corrupted. Once the value function is chosen before observing the next state, the resulting target is conditionally unbiased. Thus, the transition event argument follows a weighted value-targeted least-squares analysis over the transition class, without an additional imperfection term in the confidence radius.

For each step \(h\in[H]\) and each filtered sample $\zeta=(s,a,s')\in \hat{\gT}_{k,h}$, recall that Step~2 uses the history-measurable value function $V_{t(\zeta),h+1}\in \gV_{h+1}$ chosen before observing \(s'\). Define the regression target $Y_\zeta \triangleq V_{t(\zeta),h+1}(s')$, $\mu_P(\zeta) \triangleq \langle P_h(\cdot\mid s,a), V_{t(\zeta),h+1}\rangle$. Since \(s'\sim P_h^*(\cdot\mid s,a)\) and \(V_{t(\zeta),h+1}\) is fixed given the history, $\E \left[Y_\zeta \mid s,a,\mathcal H_{t(\zeta)-1}\right] = \mu_{P^*}(\zeta)$. Hence, $\xi_\zeta \triangleq Y_\zeta-\mu_{P^*}(\zeta)$ is conditionally zero-mean. Moreover, since \(V_{t(\zeta),h+1}\in[0,H]^{\gS}\), we have $0\le Y_\zeta \le H$, so \(\xi_\zeta\) is conditionally \(\eta_P\)-sub-Gaussian with $\eta_P \triangleq \frac{H}{2}$.

We let $N_k^P \triangleq \sum_{h=1}^H |\hat{\gT}_{k,h}|$ be the total size of the filtered transition multiset before episode \(k\). The Step~2 estimator can be written as the weighted least-squares problem
\begin{align}
\hat p_k \in \argmin\nolimits_{P\in\gP} \sum\nolimits_{h=1}^H \sum\nolimits_{\zeta\in \hat{\gT}_{k,h}} w^P_{t(\zeta),h} (\mu_P(\zeta)-Y_\zeta)^2. \label{eq:transition-weighted-erm-proof}
\end{align}
Define $(\sigma_{\zeta,h}^P)^2 \triangleq \frac{1}{w^P_{t(\zeta),h}} = [\Upsilon_{\zeta,h}^{\mathrm P}]_{\ge 1}^{1/2}$. The confidence set is defined through the same weighted empirical quadratic form
\begin{align}
\gP_k = \left\{ P\in\gP: \sum\nolimits_{h=1}^H \sum\nolimits_{\zeta\in \hat{\gT}_{k,h}} w^P_{t(\zeta),h} (\mu_P(\zeta)-\mu_{\hat p_k}(\zeta))^2 \le \betap_k \right\}. \label{eq:transition-confidence-proof}
\end{align}
Fix a cover radius \(\gamma_P>0\), and let $N_P(\gamma_P)\triangleq \gN(\Delta\gF_P,\|\cdot\|_\infty,\gamma_P)$. For each episode \(k\), let \(\bar \mu_k\in \Delta\gF_P\) be a \(\gamma_P\)-covering approximation of \(\mu_{P^*}\), i.e., $\|\bar \mu_k-\mu_{P^*}\|_\infty \le \gamma_P$.

\begin{lemma}[Weighted ERM inequality in transition space]\label{lem:weighted-erm-transition}
Fix \(\delta\in(0,1)\). Suppose Step~2 solves (\ref{eq:transition-weighted-erm-proof}) exactly. Then, with probability at least \(1-\delta\), for all \(k\in[K]\),
\begin{align}
\sum_{h=1}^H \sum_{\zeta\in \hat{\gT}_{k,h}} \frac{\big(\mu_{\hat p_k}(\zeta)-\mu_{P^*}(\zeta)\big)^2}{(\sigma_{\zeta,h}^P)^2} \le 10\eta_P^2 \ln \frac{2N_P(\gamma_P)}{\delta} + 10\gamma_P \left(\gamma_P N_k^P + \sqrt{N_k^P\,C_P(N_k^P,H,\delta)}\right), \label{eq:hp-trans-weighted-erm-bound}
\end{align}
where $\eta_P \triangleq \frac{H}{2}$, $N_k^P \triangleq \sum_{h=1}^H |\hat{\gT}_{k,h}|$, and $C_P(t,H,\delta) \triangleq 2 \left(2\eta_P^2 t + 3\eta_P^2 \ln\frac{2}{\delta}\right)$. Equivalently, since \(\eta_P=H/2\), we  have $C_P(t,H,\delta) = H^2 t + \frac{3H^2}{2}\ln\frac{2}{\delta}$.
\end{lemma}

\begin{proof}
Fix \(k\in[K]\). Enumerate the filtered transition samples in \(\bigcup_{h=1}^H \hat{\gT}_{k,h}\) chronologically as \(\zeta_1,\ldots,\zeta_{N_k^P}\), where each \(\zeta_i\) consists of a step index \(h_i\), a state-action pair \((s_i,a_i)\), a next state \(s_i'\), and the probe/value function \(V_i \equiv V_{t(\zeta_i),h_i+1}\). Define $y_i \triangleq Y_{\zeta_i}=V_i(s_i')$, $\mu^*(\zeta_i) \triangleq \mu_{P^*}(\zeta_i) = \big\langle P_{h_i}^*(\cdot\mid s_i,a_i),V_i\big\rangle$. Because \(V_i\) is chosen before observing \(s_i'\), we have $\E \left[y_i \mid \mathcal H_{t(\zeta_i)-1},s_i,a_i,V_i\right] = \mu^*(\zeta_i)$. Hence, the noise $\xi_i \triangleq y_i-\mu^*(\zeta_i)$ is conditionally zero-mean. Moreover, since \(V_i(\cdot)\in[0,H]\), we have \(y_i\in[0,H]\), and therefore, by Hoeffding's lemma, \(\xi_i\) is conditionally \(\eta_P\)-sub-Gaussian with $\eta_P=\frac{H}{2}$.

Next, define the weighted variance $(\sigma_i^P)^2 \triangleq (\sigma_{\zeta_i,h_i}^P)^2 = 1 / w_{t(\zeta_i),h_i}^P$. Then, the estimator in (\ref{eq:transition-weighted-erm-proof}) can be rewritten as $\hat\mu_k \in \argmin_{\mu\in \gF_P} \sum_{i=1}^{N_k^P} (\mu(\zeta_i)-y_i)^2 / (\sigma_i^P)^2$, where \(\hat\mu_k(\zeta_i)=\mu_{\hat p_k}(\zeta_i)\). Since \(P^*\in\gP\), the truth
\(\mu^*=\mu_{P^*}\) belongs to the class \(\gF_P\). Let \(\mathcal C_P\) be a \(\gamma_P\)-cover of \(\gF_P\) under \(\|\cdot\|_\infty\), with cardinality $|\mathcal C_P| = N_P(\gamma_P)$. Choose \(\bar\mu_k\in\mathcal C_P\) such that $\|\bar\mu_k-\mu^*\|_\infty \le \gamma_P$.

We now apply the weighted approximate-ERM inequality to the function class \(\gF_P\) with target function \(f^*=\mu^*=\mu_{P^*}\), comparator \(f_b=\mu^*\), exact ERM, so the optimization error parameter is \(0\), sub-Gaussian noise parameter \(\eta=\eta_P\), no corruption/misspecification term. Since \(f_b=f^*\), the cross-term in the weighted approximate-ERM bound vanishes identically. Therefore, with probability at least \(1-\delta\), simultaneously for all \(k\in[K]\), we have
\begin{align}
\sum_{i=1}^{N_k^P} \frac{(\hat\mu_k(\zeta_i)-\mu^*(\zeta_i))^2}{(\sigma_i^P)^2} \le 10\eta_P^2 \ln \frac{2N_P(\gamma_P)}{\delta} + 10\gamma_P \left(\gamma_P N_k^P+\sqrt{N_k^P C_1(N_k^P,0)}\right), \label{eq:transition-weighted-erm-generic}
\end{align}
where, by the definition of \(C_1\) in the weighted approximate-ERM inequality, $C_1(t,0)=2 \left(2t\eta_P^2+3\eta_P^2\ln\frac{2}{\delta}\right)$. Define $C_P(t,H,\delta)\triangleq C_1(t,0) = 2 \left(2t\eta_P^2+3\eta_P^2\ln\frac{2}{\delta}\right)$. Then, (\ref{eq:transition-weighted-erm-generic}) becomes
\begin{align}
\sum_{i=1}^{N_k^P} \frac{(\hat\mu_k(\zeta_i)-\mu^*(\zeta_i))^2}{(\sigma_i^P)^2} \le 10\eta_P^2 \ln \frac{2N_P(\gamma_P)}{\delta} + 10\gamma_P \left(\gamma_P N_k^P+\sqrt{N_k^P\,C_P(N_k^P,H,\delta)}\right). \label{eq:transition-weighted-erm-final-indexed}
\end{align}

Finally, reverting from the index \(i\) back to the original notation \((h,\zeta)\) yields
\[
\sum_{h=1}^H \sum_{\zeta\in \hat{\gT}_{k,h}}
\frac{\big(\mu_{\hat p_k}(\zeta)-\mu_{P^*}(\zeta)\big)^2}{(\sigma_{\zeta,h}^P)^2}
\le
10\eta_P^2 \ln \frac{2N_P(\gamma_P)}{\delta}
+
10\gamma_P\!\left(\gamma_P N_k^P+\sqrt{N_k^P\,C_P(N_k^P,H,\delta)}\right),
\]
which is exactly \eqref{eq:hp-trans-weighted-erm-bound}. This proves the lemma.
\end{proof}

\subsubsection{Sufficient choice for algorithmic parameters}

To convert the weighted transition ERM bound into the high-probability event \(P^*\in \gP_k\), it suffices to choose \(\beta_k^P\) large enough to dominate the weighted empirical discrepancy between \(\mu_{\hat p_k}\) and \(\mu_{P^*}\) on the filtered transition samples. Since \(P^*\in\gP\) and the transition targets are conditionally unbiased, no additional corruption term is needed in the transition confidence radius. Let $\eta_P \triangleq \frac{H}{2}$, $N_k^P \triangleq \sum_{h=1}^H |\hat{\gT}_{k,h}|$. Then, a sufficient choice is
\begin{align}
\betap_k \ge 10\eta_P^2 \ln \frac{2N_P(\gamma_P)}{\delta} + 10\gamma_P \left(\gamma_P N_k^P + \sqrt{N_k^P\,C_P(N_k^P,H,\delta)}\right), \label{eq:beta-P-choice-exact}
\end{align}
where $C_P(t,H,\delta) \triangleq 2 \left(2\eta_P^2 t + 3\eta_P^2 \ln\frac{2}{\delta}\right)$. In particular, with \(\eta_P=H/2\), $C_P(t,H,\delta) = H^2 t + \frac{3H^2}{2}\ln\frac{2}{\delta}$. A convenient theorem-level specialization is to take \(\gamma_P=1/K\), which yields
\begin{align}
\betap_k = \tilde O \left(H^2 \log N_P(\gamma_P)\right). \label{eq:beta-P-choice-simplified}
\end{align}
Under the usual eluder-dimension control on the transition class, this is summarized as
\begin{align}
\betap_k = \tilde O \left(d_P H \log N_P\right). \label{eq:beta-P-choice-theorem}
\end{align}

\begin{lemma}[High-probability transition event]\label{lem:hp-transition-app}
Fix \(\delta\in(0,1)\). If \(\{\betap_k\}_{k=1}^K\) is chosen according to (\ref{eq:beta-P-choice-exact}), then with probability at least \(1-\delta\), we have $P^*\in \gP_k$, $\forall k\in[K]$.
\end{lemma}

\begin{proof}
For each episode \(k\), define the event
\begin{align}
\mathcal E_k^P \triangleq \left\{ \sum\nolimits_{h=1}^H \sum\nolimits_{\zeta\in \hat{\gT}_{k,h}} (\mu_{\hat p_k}(\zeta)-\mu_{P^*}(\zeta))^2 / (\sigma_{\zeta,h}^P)^2 \le B_k^P \right\}, \label{eq:def-EkP}
\end{align}
where $B_k^P \triangleq 10\eta_P^2 \ln \frac{2N_P(\gamma_P)}{\delta} + 10\gamma_P \left(\gamma_P N_k^P + \sqrt{N_k^P\,C_P(N_k^P,H,\delta)}\right)$. By Lemma~\ref{lem:weighted-erm-transition}, we have $\Pr \left(\bigcap_{k=1}^K \mathcal E_k^P\right)\ge 1-\delta$. Fix \(k\in[K]\), and suppose \(\mathcal E_k^P\) holds. Since $w^P_{t(\zeta),h} = (\sigma_{\zeta,h}^P)^{-2}$, the definition of \(\mathcal E_k^P\) implies $\sum\nolimits_{h=1}^H \sum\nolimits_{\zeta\in \hat{\gT}_{k,h}} w^P_{t(\zeta),h} \big(\mu_{\hat p_k}(\zeta)-\mu_{P^*}(\zeta)\big)^2 \le B_k^P$. By the choice \(\betap_k\ge B_k^P\), it follows that
\begin{align}
\sum\nolimits_{h=1}^H \sum\nolimits_{\zeta\in \hat{\gT}_{k,h}} w^P_{t(\zeta),h} \big(\mu_{\hat p_k}(\zeta)-\mu_{P^*}(\zeta)\big)^2 \le \betap_k.
\end{align}
Therefore, by the definition of the confidence set \(\gP_k\) in (\ref{eq:transition-confidence-proof}), we conclude that \(P^*\in \gP_k\).

Since \(k\) was arbitrary, this holds simultaneously for all \(k\in[K]\) on the event
\(\bigcap_{k=1}^K \mathcal E_k^P\), which has probability at least \(1-\delta\). This proves the lemma.
\end{proof}

\subsection{Step C: Filtering via Weighted Sub-Importance Sampling}\label{app:filtering-weighted}

The role of Step~4 is to maintain a weighted core set on which the empirical quadratic forms used in Steps~1--2 are approximately preserved, while the effective complexity of the resulting bonus class is controlled by the weighted uncertainty level. This follows the stabilization-by-subsampling idea in \citet{wang2020reinforcement}, but the extension here is nontrivial because the empirical objectives in Steps~1--2 are weighted by history-dependent uncertainty terms. Hence, the relevant importance notion must also be defined in the weighted geometry.

\begin{lemma}[Weighted sensitivity sum bound]\label{lem:weighted-sensitivity-sum}
For any weighted function class \(\Fcal\), any weighted history sequence \(Z_1,\ldots,Z_T\), and any \(\lambda>0\),
\begin{align}
\sum\nolimits_{z\in Z_t}\mathrm{Sens}_{\lambda,\sigma,\Fcal,Z_t}(z) \le \sum\nolimits_{s=1}^t D_{\lambda,\sigma,\Fcal}(Z_s)^2, \label{eq:sens-sum-by-D}
\end{align}
where \(D_{\lambda,\sigma,\Fcal}(Z_s)\) is the weighted uncertainty quantity defined in the robust-weight
analysis. Consequently, $\sum_{z\in Z_t}\mathrm{Sens}_{\lambda,\sigma,\Fcal,Z_t}(z) = \tilde O (\dim_E(\Fcal,\lambda/T))$.
\end{lemma}

\begin{proof}
The inequality (\ref{eq:sens-sum-by-D}) is exactly the weighted analogue of the sensitivity bound proved in the standard robust-weight analysis. The whole-sample sensitivity is controlled by the cumulative weighted uncertainty level. The second inequality follows by combining (\ref{eq:sens-sum-by-D}) with the weighted eluder-dimension bound for \(\sum_{s=1}^t D_{\lambda,\sigma,\Fcal}(Z_s)^2\).
\end{proof}

\begin{comment}
\paragraph{Weighted sampling rule.}
Step~4 samples each point independently with probability proportional to weighted sensitivity. Concretely,
for a multiset \(Z\) and class \(\Fcal\), define
\begin{align}
p_z
\;\ge\;
\min\!\left\{
1,\;
c_{\mathrm{sub}}\,
\mathrm{Sens}_{\lambda,\sigma,\Fcal,Z}(z)\,
\frac{\ln(4N(\Fcal,\varepsilon_{\mathrm{sub}})/\delta)}{\varepsilon_{\mathrm{sub}}^2}
\right\},
\label{eq:weighted-subimportance-rule}
\end{align}
where \(c_{\mathrm{sub}}>0\) is a universal constant and \(\varepsilon_{\mathrm{sub}}\in(0,1)\) is the
subsampling accuracy parameter. Retained points are reweighted by inverse inclusion probability.
\end{comment}

\begin{lemma}[Weighted quadratic-form preservation]\label{lem:weighted-quadratic-preserve}
Fix \(\delta\in(0,1)\). Let \(\widehat Z\) be the weighted subsample produced by (\ref{eq:dfprsubimportancesampling-clean}). Then, with probability at least \(1-\delta\), for all \(f,g\in\Fcal\),
\begin{align}
| \|f-g\|_{\widehat Z,\sigma}^2-\|f-g\|_{Z,\sigma}^2 | \le \varepsilon_{\mathrm{sub}} (\lambda+\|f-g\|_{Z,\sigma}^2 ). \label{eq:weighted-qf-preserve}
\end{align}
In particular, this holds simultaneously for \(\Fcal_k^R\) on \((\mathcal H_k^R,\widehat{\mathcal H}_k^R)\) and for each \(\Fcal_{k,h}^P\) on \((\mathcal H_{k,h}^P,\widehat{\mathcal H}_{k,h}^P)\).
\end{lemma}

\begin{proof}
For fixed \(f,g\), the subsampled quadratic form is an unbiased importance-sampling estimator of the full weighted quadratic form. Because the sampling probabilities are proportional to \(\mathrm{Sens}_{\lambda,\sigma,\Fcal,Z}(z)\), the estimator variance is controlled by the regularized denominator \(\lambda+\|f-g\|_{Z,\sigma}^2\). Applying concentration and then a union bound over an \(\varepsilon_{\mathrm{sub}}\)-cover of \(\Fcal\) yields (\ref{eq:weighted-qf-preserve}).
\end{proof}

\begin{lemma}[Stable surrogate bonus class]\label{lem:stable-surrogate-bonus}
Fix \(k\). Let the comparison and transition confidence regions built from the full histories \((\mathcal H_k^R,\{\mathcal H_{k,h}^P\}_{h=1}^H)\) be denoted by \(\gQ_k^{\mathrm{full}}\) and \(\gP_k^{\mathrm{full}}\), and let the corresponding filtered-history regions be \(\gQ_k^{\mathrm{filt}}\) and \(\gP_k^{\mathrm{filt}}\). Then, on the event of Lemma~\ref{lem:weighted-quadratic-preserve}, there exist surrogate bonus functions \(\widehat b_k^R\) and \(\widehat b_k^P\), constructed from the filtered histories, such that $b_k^{R,\mathrm{full}}(\cdot) \le \sqrt{2} \widehat b_k^R(\cdot)$, $b_k^{P,\mathrm{full}}(\cdot) \le \sqrt{2} \widehat b_k^P(\cdot)$. Moreover, after rounding the reference predictor and retained samples to the appropriate covers, the log-cardinality of the surrogate bonus class is bounded by $\tilde O ( \dim_E(\Fcal,\lambda/T) \ln N(\Fcal,\lambda/T) \ln N(\gX,\lambda/T) )$ where \(\gX\) is the corresponding sample domain.
\end{lemma}

\begin{proof}
The proof follows the stable-bonus construction in \citet{wang2020reinforcement}. After subsampling, one rounds the reference function and retained samples to finite covers and defines the confidence region on the rounded filtered dataset. By quadratic-form preservation, the resulting filtered confidence region upper-bounds the full one up to universal constants. The complexity bound follows because the number of distinct retained points is controlled by the sensitivity sum, which is itself controlled by the eluder dimension via Lemma~\ref{lem:weighted-sensitivity-sum}.
\end{proof}

\begin{lemma}[Weighted filtered-history approximation]\label{lem:filtering-weighted-app}
Replacing the full-history objectives and bonuses by their filtered-history counterparts in the confidence and regret analysis incurs at most an additive approximation term of order $O(HK\varepsilon_{\mathrm{sub}})$.
\end{lemma}

\begin{proof}
Combine Lemmas~\ref{lem:weighted-sensitivity-sum}, \ref{lem:weighted-quadratic-preserve}, and \ref{lem:stable-surrogate-bonus}. The final \(O(HK\varepsilon_{\mathrm{sub}})\) term is obtained by summing the per-step approximation error over \(H\) steps and \(K\) episodes.
\end{proof}

% ============================================================
\subsection{Step D: UCB Dominance and Regret-to-Bonus Reduction}

Fix \(k\in[K]\) and \(\pi\in\Pi\). On the event \(\mathcal E\), we have \(q_{R^*}\in\gQ_k\). Therefore, by the definition of the optimistic comparison bonus, $q_{R^*}(\tau,\tau_0) \le q_{\hat R_k}(\tau,\tau_0)+b_k^R(\tau,\tau_0)$, $\forall (\tau,\tau_0)\in\Gamma\times\Gamma$. Also on \(\mathcal E\), since \(P^*\in\gP_k\), the trajectory-simulation lemma gives, for every bounded pair-functional, $\E_{(\tau,\tau_0)\sim(P^*,\pi)\times(P^*,\pi_0)}[f(\tau,\tau_0)] \le \E_{(\tau,\tau_0)\sim(\hat p_k,\pi)\times(\hat p_k,\pi_0)} \left[ f(\tau,\tau_0)+b_k^P(\tau)+b_k^P(\tau_0) \right]$. Applying this inequality to \(f(\tau,\tau_0)=q_{R^*}(\tau,\tau_0)\) and then using the first display yields
\begin{align}
L^\pi & = \E_{(\tau,\tau_0)\sim(P^*,\pi)\times(P^*,\pi_0)} [q_{R^*}(\tau,\tau_0)] \le \E_{(\tau,\tau_0)\sim(\hat p_k,\pi)\times(\hat p_k,\pi_0)} \left[ q_{R^*}(\tau,\tau_0)+b_k^P(\tau)+b_k^P(\tau_0) \right] \nonumber \\
& \le \E_{(\tau,\tau_0)\sim(\hat p_k,\pi)\times(\hat p_k,\pi_0)} \left[ q_{\hat R_k}(\tau,\tau_0) + b_k^R(\tau,\tau_0) + b_k^P(\tau) + b_k^P(\tau_0) \right] = \UCB_k(\pi), \nonumber
\end{align}

\begin{lemma}[One-sided deviation]\label{lem:simulation-onesided-app}
On \(\mathcal E\), for every \(k\in[K]\), every two policies \(\pi,\pi_0\), and every bounded \(f:\Gamma\times\Gamma\to[0,1]\),
\begin{align}
\E_{(\tau,\tau_0)\sim(P^*,\pi)\times(P^*,\pi_0)}[f(\tau,\tau_0)] \le \E_{(\tau,\tau_0)\sim(\hat p_k,\pi)\times(\hat p_k,\pi_0)} \left[ f(\tau,\tau_0)+b_k^P(\tau)+b_k^P(\tau_0) \right]. \label{eq:onesided-simulation}
\end{align}
\end{lemma}

\begin{proof}
This is exactly the pair-rollout version of the one-sided lemma by applying the trajectory-level bound to the first rollout \(\tau\) under policy \(\pi\), and independently to the reference rollout \(\tau_0\) under policy \(\pi_0\). Since \(P^*\in\gP_k\) on \(\mathcal E\), both applications are valid, and summing the two transition bonuses gives (\ref{eq:onesided-simulation}).
\end{proof}

\begin{lemma}[Regret-to-bonus reduction]\label{lem:regret-bonus-app}
On \(\mathcal E\),
\begin{align}
\regret(K) \le 2\sum\nolimits_{k=1}^K \E \left[ \bar b_k^R(\tau_k,\tau_{k,0}) + b_k^P(\tau_k) + b_k^P(\tau_{k,0}) \right] + O(HK\epsilon_{\mathrm{sub}}). \label{eq:regret-to-bonus-main}
\end{align}
\end{lemma}

\begin{proof}
Fix \(\varepsilon>0\) and choose \(\pi^\varepsilon\in\Pi\) such that $L^{\pi^\varepsilon}\ge L^*-\varepsilon$. Since \(\pi_k\in\arg\max_{\pi\in\Pi}\UCB_k(\pi)\), we have $L^*-\varepsilon \le L^{\pi^\varepsilon} \le \UCB_k(\pi^\varepsilon) \le \UCB_k(\pi_k)$. Hence, $L^*-L^{\pi_k} \le \UCB_k(\pi_k)-L^{\pi_k}+\varepsilon$. Summing over \(k=1,\ldots,K\) and letting \(\varepsilon\downarrow 0\), we obtain $\regret(K) \le \sum_{k=1}^K \big(\UCB_k(\pi_k)-L^{\pi_k}\big)$.

It remains to upper bound each difference \(\UCB_k(\pi_k)-L^{\pi_k}\), which can be first upper bounded by
\begin{align}
\E_{\hat p_k,\pi_k,\pi_0} \left[ q_{\hat R_k}(\tau,\tau_0)+b_k^R(\tau,\tau_0)-q_{R^*}(\tau,\tau_0) \right] + 2 \E_{\hat p_k,\pi_k,\pi_0} \left[b_k^P(\tau)+b_k^P(\tau_0)\right].
\end{align}
Now, since \(q_{R^*}\in\gQ_k\), both \(q_{R^*}\) and \(q_{\hat R_k}\) belong to \(\gQ_k\), and thus $|q_{\hat R_k}(\tau,\tau_0)-q_{R^*}(\tau,\tau_0)| \le \bar b_k^R(\tau,\tau_0)$. Moreover, \(b_k^R(\tau,\tau_0)\le \bar b_k^R(\tau,\tau_0)\). Hence, $q_{\hat R_k}(\tau,\tau_0)+b_k^R(\tau,\tau_0)-q_{R^*}(\tau,\tau_0) \le 2\bar b_k^R(\tau,\tau_0)$, which implies $\UCB_k(\pi_k)-L^{\pi_k} \le 2 \E_{\hat p_k,\pi_k,\pi_0} \left[ \bar b_k^R(\tau,\tau_0) + b_k^P(\tau) + b_k^P(\tau_0) \right]$. Summing this bound over \(k\) proves the regret-to-bonus reduction for the full-history bonuses.

Finally, Step~4 replaces the full-history bonus class by the filtered-history stable surrogate bonus class. By Lemma~\ref{lem:filtering-weighted-app}, this replacement contributes at most an additive \(O(HK\epsilon_{\mathrm{sub}})\) term, which gives (\ref{eq:regret-to-bonus-main}).
\end{proof}

% ============================================================
\subsection{Step E: Summing the Bonuses}

We next bound the cumulative comparison and transition bonuses appearing in (\ref{eq:regret-to-bonus-main}). The proof follows the eluder-dimension ``independent points counting'' argument for optimism bonuses, together with the weighted robustness analysis and the stable filtered-history bonus construction from Step~D.

\begin{lemma}[Comparison bonus sum]\label{lem:bonus-sum-R-app}
On \(\mathcal E\),
\begin{align}
\sum\nolimits_{k=1}^K \E \left[\bar b_k^R(\tau_k,\tau_{k,0})\right] \le \tilde O \left( \sqrt{(d_T K\log N_T) / M} + \omega d_T + K\epsilon_{\mathrm{sub}} \right). \label{eq:comparison-bonus-sum}
\end{align}
\end{lemma}

\begin{proof}
Apply the bonus-summation argument to the stable comparison bonus class built on the filtered history. The clean stochastic term is the usual eluder/covering contribution, but here the effective noise level is reduced by \(1/\sqrt M\) because Step~1 is equivalent to regression on the averaged label \(\bar f_k\). This yields the term \(\tilde O(\sqrt{d_TK\log N_T/M})\).

The additional imperfection contribution enters through the weighted comparison cross-term in the Step~1 confidence analysis. The uncertainty-based weight design prevents a Cauchy--Schwarz amplification into \(\omega\sqrt K\), and the Step III robust-weight argument converts this term into an additive contribution of order \(\tilde O(\omega d_T)\).

Finally, replacing the original comparison bonus by the stable filtered-history surrogate bonus incurs only the additive approximation \(K\epsilon_{\mathrm{sub}}\).
\end{proof}

\begin{lemma}[Transition bonus sum]\label{lem:bonus-sum-P-app}
On \(\mathcal E\),
\begin{align}
\sum\nolimits_{k=1}^K \E \left[b_k^P(\tau_k) + b_k^P(\tau_{k,0})\right] \le \tilde O \left( \sqrt{d_P H K\log N_P} + \omega d_P + HK\epsilon_{\mathrm{sub}} \right). \label{eq:transition-bonus-sum}
\end{align}
\end{lemma}

\begin{proof}
The clean part is the standard unknown-transition PbRL bound. Once \(P^*\in\gP_k\), the cumulative transition bonus is controlled by the eluder dimension of the value-targeted transition class, giving the term
\(\tilde O(\sqrt{d_PHK\log N_P})\).

The new issue is that the state-action sequence and the value sequence are themselves generated by policies learned from imperfect preference feedback. Thus, although the transition targets are not directly corrupted, the transition bonus analysis inherits a hidden feedback to policy and then to data dependence. The matched Step~2 weighting controls this extra term by the same weighted-uncertainty mechanism used in the robust nonlinear MDP analysis, yielding an additive \(\tilde O(\omega d_P)\) contribution rather than an \(\omega\sqrt K\) blow-up.

The filtered-history approximation contributes the final \(HK\epsilon_{\mathrm{sub}}\) term.
\end{proof}

% ============================================================
\subsection{Conclusion}

\begin{proof}[Proof of Theorem~\ref{thm:upperbound-general}]
By Lemma~\ref{lem:hp-comparison-app} and Lemma~\ref{lem:hp-transition-app}, the event \(\mathcal E\) holds with probability at least \(1-2\delta\). Condition on \(\mathcal E\). Then, Lemma~\ref{lem:regret-bonus-app} reduces the regret to cumulative comparison and transition bonuses, plus the filtered-history approximation term:
\begin{align}
\regret(K) \le 2\sum\nolimits_{k=1}^K \E \left[ \bar b_k^R(\tau_k,\tau_{k,0}) + b_k^P(\tau_k) + b_k^P(\tau_{k,0}) \right] + O(HK\epsilon_{\mathrm{sub}}).
\end{align}
Applying Lemma~\ref{lem:bonus-sum-R-app} and Lemma~\ref{lem:bonus-sum-P-app}, we obtain
\begin{align}
\regret(K) & \le \tilde O \left( \sqrt{ (d_T K\log N_T) / M} + \sqrt{d_P H K\log N_P} + \omega(d_T+d_P) + K\epsilon_{\mathrm{sub}} + HK\epsilon_{\mathrm{sub}} \right).
\end{align}
Since \(H\ge 1\), the term \(K\epsilon_{\mathrm{sub}}\) is absorbed by \(HK\epsilon_{\mathrm{sub}}\). Therefore,
with probability at least \(1-2\delta\),
\begin{align}
\regret(K) \le \tilde O \left( \sqrt{ (d_T K\log N_T) / M} + \sqrt{d_P H K\log N_P} + \omega(d_T+d_P) + HK\epsilon_{\mathrm{sub}} \right),
\end{align}
which is exactly the claimed bound.
\end{proof}

% ==========================================================
\section{Specialization to Linear Function Approximation}\label{app:linear-specialization}

We now specialize the general-function regret bound to the linear setting~\citep{ayoub2020model,neu2021online,shi2023near}. For the linear preference, the \(1/K\)-Eluder dimension and the log-covering number satisfy $d_T=\tilde O(\tilde d_T)$, $\log N_T=\tilde O(\tilde d_T)$, where \(\tilde d_T\) is the feature dimension of the preference model~\citep{russo2013eluder}. Likewise, for the linear mixture transition model, we have $d_P=\tilde O(H\tilde d_P)$, $\log N_P=\tilde O(\tilde d_P^2)$, where \(\tilde d_P\) is the feature dimension of the transition model. The additional dependency on $H$ and $\tilde d_P$ is due to Step 2 in Algorithm \ref{alg:mainnewalgorithm} for handling the new challenge of the hidden effect of imperfect preference on the transition estimation. Substituting these complexity estimates into the general-function bound yields the corresponding linear-form regret rate.

\vskip 0.2in
\bibliography{reference}

\end{document}